\pdfoutput=1
\documentclass{article}

\PassOptionsToPackage{numbers}{natbib}

\usepackage[main, final]{neurips_2026}

\usepackage[utf8]{inputenc} %
\usepackage[T1]{fontenc}    %
\usepackage{hyperref}       %
\usepackage{url}            %
\usepackage{booktabs}       %
\usepackage{placeins}       %
\usepackage{amsfonts}       %
\usepackage{nicefrac}       %
\usepackage{microtype}      %
\usepackage{xcolor}         %
\usepackage{graphicx}
\usepackage{subcaption}

\usepackage{amsmath}
\usepackage{amssymb}
\usepackage{mathtools}
\usepackage{amsthm}

\usepackage[capitalize,noabbrev]{cleveref}

\theoremstyle{plain}
\newtheorem{theorem}{Theorem}[section]
\newtheorem{proposition}[theorem]{Proposition}

\theoremstyle{definition}

\theoremstyle{remark}
\newtheorem{remark}[theorem]{Remark}

\usepackage[textsize=tiny]{todonotes}

\usepackage{tcolorbox}
\tcbuselibrary{skins, breakable}

\usepackage{pifont}
\usepackage{tikz}
\usetikzlibrary{positioning, arrows.meta}

\usepackage{algorithm}
\usepackage{algorithmic}
\usepackage{xspace}
\newcount\phn
\newcommand{\ph}[2]{\unskip\ {\color{red}\textbf{[#2]}}~{\color{red!20}\leaders\vrule height 6.5pt depth 2pt\hfill}\phn=1
  \loop\ifnum\phn<#1 \advance\phn by 1 \hspace{0pt}{\color{red!20}\rule[-2pt]{\linewidth}{8.5pt}}\repeat}
\usepackage{wrapfig}
\usepackage{enumitem}

\newtcolorbox{questionbox}[1][]{
  breakable,
  colback=gray!5,
  colframe=black,
  coltitle=black,
  title=#1,
  boxrule=0.5pt,
  arc=2mm,
  outer arc=2mm,
  left=6pt,
  right=6pt,
  top=6pt,
  bottom=6pt,
  fonttitle=\bfseries,
  enhanced,
  width=\textwidth,
  colbacktitle=gray!20,  %
}

\newcommand{\xb}{\mathbf{x}}
\newcommand{\xbi}{\mathbf{x}_{-i}}
\newcommand{\xbj}{\mathbf{x}_{-j}}

\newcommand{\gam}{GAM~}
\newcommand{\gamm}{GA\textsuperscript{2}M}
\newcommand{\gamms}{GA\textsuperscript{2}Ms}
\newcommand{\method}{CALM}

\title{Interpretability-by-Design with Accurate Locally Additive Models and Conditional Feature Effects}

\author{%
  Vasilis Gkolemis$^{1,2}$ \quad
  Loukas Kavouras$^{2}$ \quad
  Dimitrios Kyriakopoulos$^{6}$ \\
  \bfseries
  Konstantinos Tsopelas$^{2}$ \quad
  Dimitrios Rontogiannis$^{4}$ \quad
  Giuseppe Casalicchio$^{5}$ \\
  \bfseries
  Theodore Dalamagas$^{2}$ \quad
  Christos Diou$^{1,3}$ \\[0.5em]
  $^{1}$Department of Informatics and Telematics, Harokopio University of Athens \\
  $^{2}$IMSI, ATHENA RC \quad
  $^{3}$Archimedes AI, ATHENA RC \quad
  $^{4}$Max Planck ISS \\
  $^{5}$LMU Munich \quad
  $^{6}$National Observatory of Athens
}

\begin{document}

\maketitle

\begin{abstract}
Generalized additive models (GAMs) offer high interpretability through independent univariate feature effects but underfit when interactions are present in data.
GA$^2$Ms add selected pairwise interactions, which improves accuracy, but sacrifices interpretability.
We propose \emph{Conditionally Additive Local Models} (CALMs), a new model class that narrows this trade-off: it stays close to GAMs in interpretability and close to GA$^2$Ms in accuracy.
CALMs allow multiple univariate shape functions per feature, each active in different regions of the input space.
These regions are defined independently for each feature as simple logical conditions (thresholds) on the features it interacts with.
As a result, effects remain locally additive while varying across subregions to capture interactions.
We further propose a principled distillation-based training pipeline that identifies homogeneous regions with limited interactions and fits interpretable shape functions via region-aware backfitting.
Experiments on diverse classification and regression tasks show that CALMs never underperform GAMs, often beat them and achieve accuracy broadly comparable to GA$^2$Ms, while preserving the univariate auditability that GA$^2$Ms forfeit.
\end{abstract}

\section{Introduction}
\label{sec:introduction}
In high-stakes decision making, models must provide not only reliable predictions but also human-understandable explanations \cite{murdoch2019interpretable, doshi2017towards}. 
This has motivated \emph{interpretable-by-design} models, whose structure permits direct inspection of their behavior \cite{rudin2019stop, molnar2020interpretable}.
Interpretability, however, is not a yes-or-no property but a continuum: interpretable-by-design models range from more to less interpretable, and gains in predictive accuracy often come at the cost of greater structural complexity \cite{lipton2018mythos, Hammann2025xVSacc}.

\begin{figure*}[t]
  \centering
  \includegraphics[width=\textwidth]{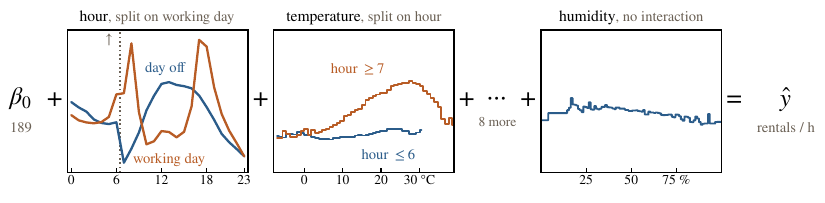}
  \caption{A CALM fitted on the Bike Sharing dataset: the prediction sums one plot per feature, each plot holding one curve per region. The hour splits on the working day; the temperature splits on the hour, mattering by day and barely at night; the humidity interacts with no feature and keeps one curve. The dotted line on the hour's plot marks where the temperature's and the year's regions switch and its arrow gives the sign of the hidden jump; Section~\ref{subsec:interpretability-properties} explains it in depth.}
  \label{fig:concept-image}
\end{figure*}

This trade-off is evident in GAMs~\cite{lou2012intelligible} and their extension, \gamms~\cite{lou2013accurate, caruana2015intelligible}.
GAMs model the prediction as a sum of univariate effects, which makes them highly interpretable but blind to feature interactions which limits their accuracy.
\gamms~gain accuracy by adding selected pairwise interactions, but the interaction terms cost them two interpretability properties \cite{radenovic2022neural, xiong2026neural}:
(i) a prediction no longer splits into one contribution per feature, because an interaction term belongs to two features jointly, and
(ii)~global auditing is hard to do visually: a feature's effect is spread over its univariate and bivariate terms, so no single plot shows how the feature affects the prediction (see Section~\ref{subsec:interpretability-properties}).

We introduce \emph{Conditionally Additive Local Models (CALMs)}, a model class that narrows this trade-off: it stays close to GAMs in interpretability and close to \gamms~in accuracy.
The core idea is to learn \emph{conditional feature effects}~\cite{hastie1993varying, berger2019tree}: several univariate effects per feature, each active in a region of the input space where the feature is nearly additive~\cite{hooker2004discovering}, i.e., barely interacts with other features.
The regions are defined by threshold conditions on the features it interacts with.
Figure~\ref{fig:concept-image} shows a CALM fitted on the Bike Sharing dataset~\cite{fanaee2014event} and Figure~\ref{fig:hero} compares it with a GAM and a \gamm.

In this example, the feature \texttt{hour} interacts with \texttt{workingday}, so a CALM draws its effect as two curves, one for working and one for non-working days.
On working days rentals peak at the morning and evening commute, whereas on non-working days they rise to a single broad peak around midday.
The CALM thus captures the interaction while keeping (i) one contribution per feature and (ii) one plot per feature, unlike a GAM, which averages the two curves into one and loses accuracy, and a \gamm{}, which expresses the interaction as a pairwise surface and loses both interpretability properties.

\begin{figure*}[t]
  \centering
  \begin{subfigure}{0pt}\phantomsubcaption\label{fig:hero-gam}\end{subfigure}%
  \begin{subfigure}{0pt}\phantomsubcaption\label{fig:hero-calm}\end{subfigure}%
  \begin{subfigure}{0pt}\phantomsubcaption\label{fig:hero-ga2m}\end{subfigure}%
  \includegraphics[width=\linewidth]{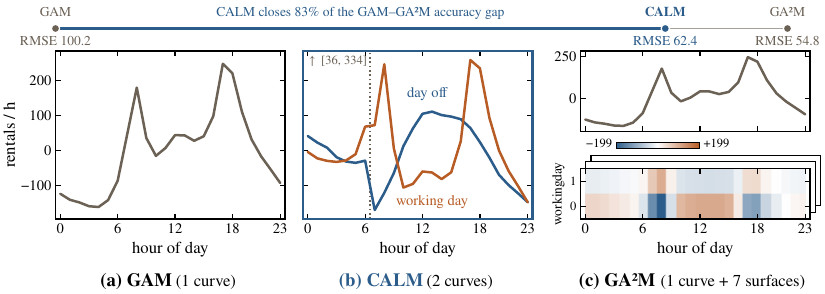}
  \caption{The effect of the hour on bike rentals (Bike Sharing dataset). Top: test RMSE (Table~\ref{tab:regr-rmse}); CALM closes 83\% of the gap from GAM to \gamm.
    (\subref{fig:hero-gam}) GAM: one curve. The interaction with the working day is averaged away, which limits accuracy.
    (\subref{fig:hero-calm}) CALM: one curve for working days and one for days off (at $d_{\max} = 1$). The dotted line mark a hidden jump (see Section~\ref{subsec:interpretability-properties} for more details).
    (\subref{fig:hero-ga2m}) \gamm: a curve and seven 2D surfaces (only one shown here) for decoding the effect of the hour, which limits its interpretability.}
  \label{fig:hero}
\end{figure*}

We fit CALMs with a three-step distillation pipeline:
(1) train a black-box reference model that captures the interactions in the data;
(2) for each feature, grow a CART-style tree~\cite{breiman1984cart, herbinger2022REPID} that splits the input space into regions of reduced interaction, then greedily keep the trees that matter most for reproducing the reference model's predictions;
(3) fit one curve per region with a region-aware backfitting~\cite{buja1989linear}.
Two parameters bound what a reader must inspect: the tree depth $d_{\max}$ caps how many curves each feature has and $K$ caps how many interactions the model captures in total.

\textbf{Contributions.}
(i) CALMs, an interpretable-by-design model class built on conditional feature effects (Section~\ref{subsec:model-overview});
(ii) a distillation-based training algorithm (Section~\ref{subsec:training});
(iii) a formal analysis showing that CALMs keep the two interpretability properties that \gamms~lose (Section~\ref{subsec:interpretability-properties});
(iv) an evaluation on 24 real-world regression and classification datasets, in which CALMs never lose to a GAM baseline and match or beat EB$^2$M on more than half (Section~\ref{sec:real-datasets});
(v) evidence that CALMs stay small and stable: about eight curves beyond a GAM, growing with tree depth as accuracy rises (Section~\ref{sec:complexity}), and the same structure for most features across resamples (Section~\ref{sec:stability});
(vi) a within-subjects user study ($N=25$), in which participants answered more accurately from CALM plots than from \gamm{} heatmaps (Section~\ref{sec:user-study}).
Code: \url{https://github.com/givasile/calm-paper}.

\section{Background and Related Work}
\label{sec:related-work}
Decision trees~\cite{breiman1984cart}, rule lists~\cite{letham2015stroke, angelino2018learningcertifiablyoptimalrule} and prototype-based classifiers~\cite{chen2019looks, bien2012prototype} are typical interpretable-by-design models; among them, GAMs~\cite{hastie1990generalized, wood2017generalized} are popular for their simple global interpretability.

GAMs model the output as $g(\mathbb{E}[y|\mathbf{x}]) = \beta_0 + \sum_{i} f_i(x_i)$, with univariate shape functions $f_i$, intercept $\beta_0$ and link function $g$.
They differ mainly in how $f_i$ are learned: penalized splines~\cite{wood2017generalized} give smooth curves, boosted bagged shallow trees~\cite{lou2012intelligible, nori2019interpretml} give step functions that follow sharp changes~\cite{chang2021interpretable}, and neural variants learn one network per feature~\cite{agarwal2021neural}, a set of basis functions shared by all features~\cite{radenovic2022neural}, or a boosted ensemble of extreme learning machines~\cite{KRAUS2024IGANN}.

\gamms~add pairwise terms, $g(\mathbb{E}[y|\mathbf{x}]) = \beta_0 + \sum_{i} f_i(x_i) + \sum_{(i,j)\in\mathcal{I}} f_{ij}(x_i, x_j)$, over a set $\mathcal{I}$ of selected pairs, and differ in how they choose $\mathcal{I}$ and fit the terms~\cite{tsang2018detecting, hu2022using}.
GA$^2$M~\cite{lou2013accurate} fits boosted trees and adds pairs greedily, ranked by how much each reduces the residual loss;
GAMI-Net~\cite{yang2020gami} fits one network per term and considers only pairs with at least one already-selected main effect;
GRAND-SLAMIN~\cite{ibrahim2023grandslamin} fits soft trees and selects pairs end to end through differentiable gates, optionally under hierarchy constraints;
and NODE-GA$^2$M~\cite{chang2022node} fits differentiable oblivious trees, each restricted to at most two features.

CALM draws inspiration from \emph{regional effect methods}, post-hoc explainability methods that extend \emph{global effect plots}.
Global effect plots, such as Partial Dependence Plots (PDP)~\cite{friedman2001greedy}, Accumulated Local Effects (ALE)~\cite{apley2020visualizing} and SHAP Dependence Plots (SHAP-DP)~\cite{lundberg2020local}, explain a black-box model $f$ by decomposing it into $d$ one-dimensional plots, one per feature $x_i$.
When $f$ contains strong interactions, these plots average heterogeneous local effects into one curve and can mislead, a phenomenon known as aggregation bias~\cite{mehrabi2021survey, herbinger2022REPID, baniecki2024robustness, gkolemis2023rhale, goldstein2015peeking, molnar2022general}.

Regional effect methods~\cite{gkolemis2024fast} address this by splitting the input space with a tree into subregions where interactions are weak~\cite{scholbeck2024marginal}, and drawing each feature's plot separately in every subregion.
REPID~\cite{herbinger2022REPID} grows one tree per feature, while GADGET~\cite{herbinger2024GADGET}, FD-Tree~\cite{laberge2024fdtree} and GRANITE~\cite{herbinger2026granite} grow one tree shared by several features.
CALM adopts GADGET's heterogeneity criterion and REPID's per-feature trees to find its regions (Step~1(b), Section~\ref{subsec:training}).
Per-feature trees let each feature split only on the features it interacts with~\cite{berger2019tree, burgin2017coefficient}, whereas a shared tree imposes the same splits on every feature~\cite{zeileis2008model}.
In Figure~\ref{fig:concept-image}, a shared tree would split the humidity, which interacts with no feature, on the working day only because the hour needs it.
Finally, we clarify that GADGET, REPID and all regional effect methods are not competitors of CALM: they explain a fitted model post hoc, whereas a CALM is itself a predictor.

CALM also relates to model distillation~\cite{hinton2015distilling, tan2018distill} and surrogate models~\cite{guidotti2018survey}, which train an interpretable student to mimic a black-box teacher, locally~\cite{ribeiro2016should} or globally~\cite{frosst2017distillingneuralnetworksoft, guidotti2018survey}.
However, instead of explaining the teacher, a CALM uses it only to find where interactions are weak, and then fits its curves to the data.
In a rough analogy, CALM serves as the interpretable-by-design student that replaces the black-box teacher as the final predictor~\cite{henckaerts2022when, gosiewska2021simpler}.
SLIM~\cite{hu2020surrogate} and other model-based trees~\cite{herbinger2023leveraging} fit a simple model in each leaf of a single tree, mainly as surrogates that explain a black box, so all features share one partition.
Regionally additive models~\cite{gkolemis2023regionally} also build an interpretable model from regional effects, but do not characterize its interpretability or evaluate it systematically.

\section{CALM: Conditionally Additive Local Model}
\label{sec:ram}
CALM captures feature interactions via \emph{conditional feature effects}, a set of univariate shape functions per feature, 
each active in a different region of the input space.

\label{subsec:model-overview}

\textbf{Model formulation.}
Let $\mathbf{X}=(X_1,\ldots,X_d)$ be a random vector with distribution $P_{\mathbf{X}}$ on $\mathcal{X}=\mathcal{X}_1 \times \cdots \times \mathcal{X}_d\subseteq\mathbb{R}^d$, and $\xb=(x_1,\ldots,x_d)$ a point of $\mathcal{X}$.
We use the subscript $-i$ to denote quantities excluding the $i$-th feature (e.g., $\mathbf{X}_{-i}$ defined on $\mathcal{X}_{-i}$). Let also $Y$ be the output random variable, with $Y\in\mathbb{R}$ for regression and $Y\in\{0,1\}$ for binary classification.
CALM is then defined as
\begin{equation}
    \label{eq:CALM}
    g \left ( \mathbb{E}[Y \mid \mathbf{X}=\xb ] \right ) = \beta_0 + \sum_{i=1}^d f_i^{(r_i(\xbi))}(x_i)
\end{equation}
where 
\( \beta_0 \in \mathbb{R} \) is an intercept and
\( g \) is a link function (identity for regression and logit for classification).
For each feature $i$, CALM learns a set of univariate shape functions \( \{ f_i^{(r)} \}_{r=1}^{R_i} \).
At prediction time, a region selection function $r_i: \mathcal{X}_{-i} \to \{1, \ldots, R_i\}$ chooses the specific shape function for $x_i$ based on the context of the other features $\xbi$.
This allows the effect of $x_i$ to vary across regions, effectively modeling feature interactions, while maintaining univariate interpretability within each region.
We denote the prediction of a CALM by $\hat y(\xb)=f_{\texttt{CALM}}(\xb)=g^{-1}\big(\beta_0+\sum_{i=1}^d f_i^{(r_i(\xbi))}(x_i)\big)$.

\textbf{Partition trees.}
Each selection function \( r_i \) partitions $\mathcal{X}_{-i}$, i.e., the input space excluding $x_i$, into regions $\mathcal{P}_i := \{\mathcal{R}_i^{(r)}\}_{r=1}^{R_i}$, independently for each feature $i$.
The goal of these partitions is to minimize interactions between $x_i$ and other features within each region, ensuring the effect of $x_i$ is accurately captured by a single univariate shape function.
Each partition is represented by its own binary decision tree $T_i$ of maximum depth $d_{\max}$, built over $\xbi$. 
Internal nodes apply axis-aligned splits, i.e., inequality thresholds \(x_j < \tau\) or $x_j \geq \tau$ for continuous features, and equality tests \( x_j = \tau \) or \(x_j \neq \tau\) for categorical ones. 
Each leaf of the tree defines a region \(\mathcal{R}_i^{(r)}\).
Formally, each region is defined by a conjunction of at most $m_i^{(r)} \leq d_{\max}$ rules:
\begin{equation}
\label{eq:region-definition}
\mathcal{R}_i^{(r)} = \left\{ \xbi  \;\middle|\; \textstyle\bigwedge_{k=1}^{m_i^{(r)}} (x_{j_k} \;\text{op}_k\; \tau_k)  \right\}
\end{equation}
where 
\( j_k \in \{1, \dots, d\} \setminus \{i\} \) indexes an interacting feature, 
\( \text{op}_k \in \{<, \geq, =\,, \ne\} \) is a comparison operator,
and
 \( \tau_k \) is a threshold or a category for categorical features.

\textbf{Centering.}
The shape functions are centered at feature level.
In a GAM, each feature has one curve with zero mean over the data, and the intercept $\beta_0$ is the average prediction (on the scale of $g$).
In a CALM, a feature has several curves, so we require them to jointly average to zero:
\begin{equation}
\label{eq:centering-main}
\mathbb{E}\big[f_i^{(r_i(\mathbf{X}_{-i}))}(X_i)\big]
= \sum_{r=1}^{R_i} \mathbb{P}\big(\mathbf{X}_{-i}\in\mathcal{R}_i^{(r)}\big)\,
\mathbb{E}\big[f_i^{(r)}(X_i)\,\big|\,\mathbf{X}_{-i}\in\mathcal{R}_i^{(r)}\big] = 0
\quad \text{for all } i.
\end{equation}
We deliberately avoid centering each curve separately, as this would erase meaningful differences, e.g., the gap between the working-day and day-off curves of the hour in Figure~\ref{fig:concept-image}.

\textbf{High-order interactions.}
As each region conditions on up to $d_{\max}$ features, CALM captures interactions of up to $d_{\max}+1$ features.
For example, at $d_{\max}=2$ the hour of the Bike Sharing data has one curve for working days of the second year ($\mathtt{workingday}=1 \wedge \mathtt{yr}=1$, Figure~\ref{fig:ladder}).
This three-feature interaction is modeled by a univariate curve, whereas a \gamm{} cannot express it at all, and a model with three-way terms would need a 3-D surface.

\textbf{Hyperparameters.}
Increasing $d_{\max}$ typically improves accuracy (Section~\ref{sec:complexity}), but at the cost of interpretability.
To balance this trade-off, CALM exposes two hyperparameters: $d_{\max}$ (maximum tree depth), which bounds the number of shape functions per feature, $R_i \leq 2^{d_{\max}}$, and $K$, which bounds the total number of \emph{conditional interactions} across all trees.
For instance, the three-feature interaction of the Bike Sharing example above counts as two conditional interactions: (hour, working day) and (hour, year).
Two thresholds, $\epsilon$ and $\gamma$, control how readily a split and a tree are accepted (Step~1 of Algorithm~\ref{alg:calm-training}).
In all experiments, we set $d_{\max} = 2$, $\epsilon = 0.05$, $\gamma = 0.01$, and leave $K$ unconstrained.

\textbf{Expressivity.}
CALM learns at most $2^{d_{\max}}$ curves per feature, yet, since every instance selects one region per feature, these curves combine into up to $\prod_{i} R_i \leq 2^{d\cdot d_{\max}}$ distinct additive models.
Fitting these models one by one would split the data into up to $\prod_{i} R_i$ cells, leaving most of them with too few instances. CALM instead shares each curve across all the models whose region it covers, so every curve is still fitted on enough instances: about a $1/R_i$ share of the data for balanced regions.
By contrast, a single tree shared by all features would yield at most $2^{d_{\max}}$ additive models.

\subsection{Training algorithm for fitting CALM}
\label{subsec:training}

Given a dataset \( \mathcal{D} = \{(\xb^{(k)}, y^{(k)})\}_{k=1}^N \), we fit a CALM
by minimizing the empirical risk
\( \mathbb{E}_{(\mathbf{X}, Y) \sim \hat{P}_{\mathcal{D}}}\left [ \mathcal{L}(Y, \hat{y}(\mathbf{X})) \right ] \),
where $\mathcal{L}$ is a loss function
and
$\hat{P}_{\mathcal{D}}$ the empirical distribution over $\mathcal{D}$.
This requires estimating
(i) the region selection functions \( r_i \, \forall i \), i.e., a partition tree $T_i$ for each feature, and
(ii) one shape function per region: $\{f_i^{(r)}\}_{r=1}^{R_i}$ and the intercept $\beta_0$.
We solve this problem in two sequential steps summarized in Algorithm~\ref{alg:calm-training}.
\begin{algorithm}[tb]
\caption{Training a CALM model}
\label{alg:calm-training}
\begin{algorithmic}[1]
\renewcommand{\algorithmicensure}{\textbf{Output:}} %
\REQUIRE Training data \( \mathcal{D} = \{(\xb^{(k)}, y^{(k)})\}_{k=1}^N \)
\ENSURE CALM predictor \( f_{\texttt{CALM}}(\xb) \)
\STATE \textbf{Step 1:} Find the regions.
\STATE \quad \textbf{1(a):} Train a reference model \( f_{\texttt{ref}} \) on \( \mathcal{D} \).
\STATE \quad \textbf{1(b):} Propose a tree \( T_i \) for each feature with \( H_i \) at or above the median. \COMMENT{Eq.~\eqref{eq:heterogeneity}, Alg.~\ref{alg:fit-tree}}
\STATE \quad \textbf{1(c):} Select the trees \( \mathcal{S} \) greedily on the surrogate \( R^2 \). \COMMENT{Eq.~\eqref{eq:surrogate-r2}, Alg.~\ref{alg:select-trees}}
\STATE \textbf{Step 2:} Fit the shape functions \( \{ \{f_i^{(r)}\}_{r=1}^{R_i} \}_{i=1}^d \) and \( \beta_0 \). \COMMENT{Alg.~\ref{alg:gb-calm}}
\end{algorithmic}
\end{algorithm}

\textbf{Step 1: Find the regions.}
Step 1 searches for a partition tree $T_i$ per feature such that, within each region, the effect of $x_i$ is well captured by a single univariate shape function, i.e., interactions are locally weak.
It has three substeps:
Step~1(a) trains a reference model,
Step~1(b) \emph{proposes} a tree per feature, and
Step~1(c) \emph{selects} which trees enter the model.

\textit{Step 1(a) Train a reference model.}
We train a high-capacity black-box predictor \( f_{\texttt{ref}} : \mathcal{X} \to \mathbb{R} \) on \( \mathcal{D}\), as a proxy for \( \mathbb{E}[Y | \mathbf{X} ] \).
Any accurate predictor can be used, such as gradient-boosted trees~\cite{chen2016xgboost} (our default), neural networks, random forests~\cite{breiman2001random} or foundation models like TabPFN \cite{hollmann2022tabpfn}.

\textit{Step 1(b) Propose a tree per feature.}
To find such regions, we use the interaction-related heterogeneity of GADGET~\cite{herbinger2024GADGET}.
Feature effect methods such as PDP or ALE explain a model like $f_{\texttt{ref}}$ through local effects \(h(x_i, \xbi^{(k)})\), the contribution of $x_i$ to instance $\xb^{(k)}$, and average them into a global effect $f_i(x_i)$.
For PDP, our default, the local effect is the centered ICE curve~\cite{goldstein2015peeking}, \(h(x_i, \xbi^{(k)}) = f_{\texttt{ref}}(x_i, \xbi^{(k)}) - \mathbb{E}_{X_i}\left[f_{\texttt{ref}}(X_i, \xbi^{(k)})\right]\).
Heterogeneity is the spread of the local effects around their average: high where $x_i$ interacts with other features, low in near-additive regions~\cite{heiss2026sails}.
In a region \(\mathcal{R} \subseteq \mathcal{X}_{-i}\) (e.g., \(\mathcal{R}_i^{(r)}\)), the pointwise heterogeneity of $x_i$ is
\begin{equation}
\begin{aligned}
H_i^{\mathcal{R}}(x_i) &= \mathbb{E}_{\mathbf{X}_{-i}\mid \mathbf{X}_{-i} \in \mathcal{R}}\left[\left ( h(x_i , \mathbf{X}_{-i}) - \mu_i^{\mathcal{R}}(x_i) \right )^2\right]
\end{aligned}
\label{eq:pointwiseh}
\end{equation}
where $\mu_i^{\mathcal{R}}(x_i) = \mathbb{E}_{\mathbf{X}_{-i}\mid \mathbf{X}_{-i} \in \mathcal{R}}\left[h(x_i , \mathbf{X}_{-i})\right]$ and the feature-level heterogeneity is
\begin{equation}
    H^{\mathcal{R}}_i=\sqrt{\mathbb{E}_{X_i|\mathbf{X}_{-i}\in \mathcal{R}}\left[H_i^{\mathcal{R}}(X_i)\right]}
    \label{eq:heterogeneity}
\end{equation}
The square root places $H^{\mathcal{R}}_i$ in the units of the output: it is the typical deviation of a local effect from the average effect of $x_i$ within $\mathcal{R}$.
ALE- and SHAP-DP-based variants are valid alternatives \citep{herbinger2022REPID, herbinger2024GADGET}, and RHALE \citep{gkolemis2023rhale} is preferable under strong feature correlations (Appendix~\ref{subsec:heterogeneity-measures}).
We grow each $T_i$ greedily, in a CART-style fashion, with $H_i^{\mathcal{R}}$ as the criterion.
A tree is proposed only for features whose global heterogeneity $H_i^{\mathcal{X}_{-i}}$ is at or above the median; the others keep a single shape function.
At each node, the split over the conditioning features $x_j$, $j \neq i$, that most reduces heterogeneity is accepted if the relative reduction exceeds $\epsilon$; otherwise the node becomes a leaf (Appendix~\ref{subsec:partition-tree}).
CART is one choice of partitioner and any other rule learners such as BRL~\citep{letham2015stroke} could replace it.

\textbf{Theoretical justification.}
We show that, under specific assumptions, the error of CALM in approximating $f_{\texttt{ref}}$ is at most $d$ times the sum of the expected squared feature heterogeneities (Proposition~\ref{prop:step2}).
Step~1(b) lowers a computable proxy of this bound, the PDP heterogeneity of $f_{\texttt{ref}}$; whether it also lowers the predictive error is an empirical question, addressed in Section~\ref{sec:real-datasets}.

\textit{Step 1(c) Select the trees to keep.}
The trees of Step~1(b) are proposed one feature at a time and are often redundant: an interaction between $x_i$ and $x_j$ raises the heterogeneity of both features, so both trees split on each other, although one suffices.
Keeping every candidate would double-count such interactions and increase the number of curves.
We therefore keep the trees that most improve an additive surrogate of $f_{\texttt{ref}}$.
For a set $\mathcal{S}$ of accepted trees, let $\mathcal{R}_i(\xbi)$ be the region of $T_i$ that $\xbi$ falls in if $i \in \mathcal{S}$, and $\mathcal{X}_{-i}$ otherwise.
The surrogate sums the average local effects $\mu_i^{\mathcal{R}}$ of Eq.~\eqref{eq:pointwiseh} on these regions, and its fidelity to $f_{\texttt{ref}}$ is
\begin{equation}
\label{eq:surrogate-r2}
g_{\mathcal{S}}(\xb) = c_{\mathcal{S}}(\xb) + \sum_{i=1}^d \mu_i^{\mathcal{R}_i(\xbi)}(x_i),
\qquad
R^2(\mathcal{S}) = 1 - \frac{\mathbb{E}\left[\left(f_{\texttt{ref}}(\mathbf{X}) - g_{\mathcal{S}}(\mathbf{X})\right)^2\right]}{\mathrm{Var}\left[f_{\texttt{ref}}(\mathbf{X})\right]},
\end{equation}
where each $\mu_i^{\mathcal{R}}$ is centered over the instances of $\mathcal{R}$, $c_{\mathcal{S}}$ is a piecewise-constant offset over the regions of the trees in $\mathcal{S}$, fitted by least squares, and the expectations are over the training data.
The curves are those already computed in Step~1(b); only $c_{\mathcal{S}}$ is fitted.

Starting from the GAM surrogate ($\mathcal{S} = \emptyset$), each round accepts the candidate with the largest gain in $R^2(\mathcal{S})$ on top of the trees already accepted, provided the gain is at least $\gamma$; selection stops when no candidate clears $\gamma$ or when the next tree would exceed $K$ conditional interactions (Appendix~\ref{subsec:selection}, Algorithm~\ref{alg:select-trees}).
A redundant tree adds almost nothing on top of its partner and is discarded; this is what keeps CALMs small (Section~\ref{sec:real-datasets}), and $K$ only truncates the same sequence.
We stress that $R^2(\mathcal{S})$ measures in-sample fidelity to $f_{\texttt{ref}}$, not predictive accuracy, which is assessed on held-out data in Section~\ref{sec:real-datasets}.
By construction, CALM cannot discover interactions that $f_{\texttt{ref}}$ has not captured; where $f_{\texttt{ref}}$ is near-additive, CALM remains a GAM.

\textbf{Step 2: Fit the shape functions.}
Given the trees, we estimate the region-specific shape functions $\{f_i^{(r)}\}$ and $\beta_0$ by minimizing the empirical risk.
We use a modified gradient boosting~\cite{buhlmann2007boosting} in which each iteration updates a single shape function $f_i^{(r)}$, using only the instances whose $\xbi$ falls in $\mathcal{R}_i^{(r)}$ (Algorithm~\ref{alg:gb-calm}).

Although each update is restricted to a single region, the resulting optimization problem is inherently \emph{coupled}.
The regions are defined separately for each feature, so the subsets of samples used to update different shape functions generally overlap.
Consequently, updating one $f_i^{(r)}$ changes the residuals seen by all other shape functions.
Therefore it can be understood as a coordinated optimization of a single additive predictor with region-gated components.

The following proposition characterizes its population target and the convergence of its idealized exact version, which gradient boosting approximates in practice; the proof is in Appendix~\ref{app:step3-theory}.

\begin{proposition}[Optimality and convergence]
\label{prop:step3-main}
Let $m(\xb) = \mathbb{E}[Y \mid \mathbf{X}=\xb]$ denote the true regression function. Assume regression with squared loss, fixed partition trees $\{T_i\}_{i=1}^d$ (as learned by Step~1) and $\mathbb{E}[Y^2]<\infty$.
Let $\mathcal{H}(\{T_i\})$ denote the fixed-tree CALM class (Appendix~\ref{app:step3-theory}), and assume
$\mathcal{H}(\{T_i\})\subset L_2(P_{\mathbf{X}})$ is nonempty, closed, and convex.
Then:
\begin{enumerate}
    \item \emph{Optimality}: 
    Any 
    $s^\star \in \arg\min_{s\in\mathcal{H}(\{T_i\})} \mathbb{E}[(Y-s(\mathbf{X}))^2]$
    is the $L_2(P_{\mathbf{X}})$-best approximation of $m$ within $\mathcal{H}(\{T_i\})$.
    \item \emph{Convergence}: The idealized exact cyclic regional backfitting converges to an empirical risk minimizer over $\mathcal{H}(\{T_i\})$.
\end{enumerate}
\end{proposition}

\subsection{Interpretability of a CALM} \label{subsec:interpretability-properties}
CALMs offer interpretability similar to GAMs, as both require inspecting $d$ univariate plots; one per feature.
CALMs are slightly more complex: unlike GAMs, each plot can contain up to $2^{d_{\max}}$ curves corresponding to a different region of the input space and with interaction-induced discontinuities marked by vertical lines.
However, because regions are interpretable (specified by simple threshold conditions), these region-specific curves provide explanations which are suitable for model auditing and decision support.
For example:
\textit{The effect of the $\mathtt{hour}$ on $\mathtt{rentals}$ follows the working-day curve of Figure~\ref{fig:hero} on a working day and the day-off curve on a day off.}

\textbf{Interpreting a CALM plot.}
In Figure~\ref{fig:hero-calm} each curve gives the contribution of the $\mathtt{hour}$ to $\mathtt{rentals}$ in a specific region;
the orange curve on a working day and the blue curve on a day off.
For example, at 8\,am the contribution is $+246$ (working day) or $-122$ (day off).
The plots also illustrate how altering \( x_i \) to \(x_i \rightarrow x_i + \Delta x \) impacts the prediction.
Vertical dotted lines mark points of a hidden discontinuity which is due to $x_i$ participating as an interaction term for another feature. As shown in Figure~\ref{fig:concept-image}, the effect of the $\mathtt{temperature}$ is conditioned by $\mathtt{hour} \leq 6$ and $\mathtt{hour} \geq 7$ (as is the year's), therefore in Figure~\ref{fig:hero-calm} we observe a vertical line at $6.5$.
If a change in $x_i$ does not cross a vertical line, the change in the output \( (\Delta \hat{y}) \) equals the curve difference \( (\Delta f_i )\).
Crossing a line signifies a hidden jump, in the range $[\alpha, \beta]$, so 
$ \Delta f_i + \alpha \leq \Delta \hat{y} \leq \Delta f_i + \beta$.
Arrows provide a fast understanding of the jump:
\(\uparrow\) means $\Delta \hat{y} > \Delta f_i$,
\(\downarrow\) means $\Delta \hat{y} < \Delta f_i$,
\(\updownarrow\) means it depends.
The jump at $6.5$ is $[36, 334]$, always \(\uparrow\).
Below, we outline three crucial interpretability properties, along with discussion about their satisfiability by GAM, \gamm, and \method; Figure~\ref{fig:props} answers each on the model of Figure~\ref{fig:hero}. Complete proofs are provided in Appendix~\ref{app:interpretability}.

\begin{figure*}[t]
  \centering
  \begin{subfigure}{0pt}\phantomsubcaption\label{fig:props-p1}\end{subfigure}%
  \begin{subfigure}{0pt}\phantomsubcaption\label{fig:props-p2}\end{subfigure}%
  \begin{subfigure}{0pt}\phantomsubcaption\label{fig:props-p3}\end{subfigure}%
  \includegraphics[width=\linewidth]{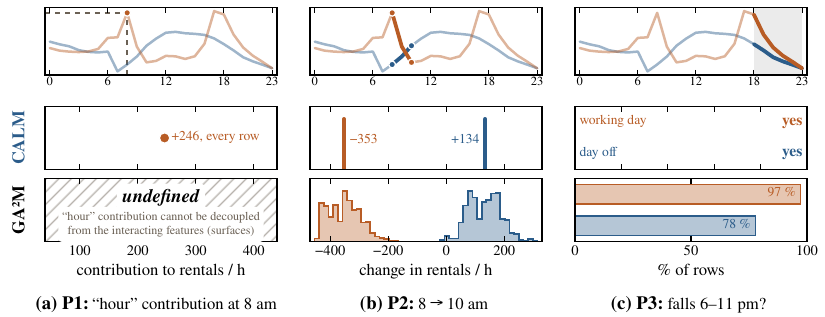}
  \caption{The three properties on the model of Figure~\ref{fig:hero}, each asked on every row of the data. Top: CALM's curves, the question marked. CALM's answers are read by eye; the \gamm's are computed, each of the 17{,}379 rows through its eight terms with the hour.
    (\subref{fig:props-p1}) The hour's contribution at 8\,am on a working day: CALM reads $+246$ off the working-day curve, the same on every row; the \gamm's is undefined, since the hour's contribution cannot be decoupled from the interacting features (its seven surfaces).
    (\subref{fig:props-p2}) $\hat{y}(10\,\text{am}) - \hat{y}(8\,\text{am})$: CALM $-353$ on working days and $+134$ on days off, exact on every row since no dotted line is crossed; the \gamm{} spreads from $-430$ to $-257$ and from $+54$ to $+215$ (5th--95th percentile).
    (\subref{fig:props-p3}) Does the prediction fall every hour from 6\,pm to 11\,pm? CALM: yes on working days and days off, since both curves fall and cross no dotted line; the \gamm: on 97\% of working days and 78\% of days off.}
  \label{fig:props}
\end{figure*}

\textbf{P1. Local Feature Contribution:} What is the contribution of each feature to the prediction?\\
\textit{Formally:} Given an input $\xb$, how much does each $x_i$ contribute to $\hat y(\xb)$?\\
The explanation is \emph{local}: it concerns a specific input $\xb$.
In GAM, the contribution is \( f_i(x_i) \).
In \gamm, the contribution is neither explicit nor unique: the pairwise surfaces must be split post hoc, e.g.\ via SHAP~\cite{lundberg2017unified} or LIME~\cite{ribeiro2016should} (Figure~\ref{fig:props-p1}).
In \method, the contribution is \( f_i^{(r_i(\xbi))}(x_i) \).

\textbf{P2. Regional Feature Sensitivity:} How does changing $x_i$ change the prediction? 
\textit{Formally:} Given $x_i$ and $\Delta x>0$, what is $\Delta \hat{y} = \hat y(\xb +  \mathbf{e}_i \Delta x )-\hat y(\xb)$, assuming only $x_i$ is perturbed and $\xbi$ is fixed? 
The explanation is \emph{regional}: it characterizes the effect of $x_i$ over $[x_i, x_i + \Delta x]$.
In GAM, the change is $\Delta \hat{y} = f_i(x_i+\Delta x)-f_i(x_i)$.
In \gamm, the change depends on every feature that interacts with $x_i$, one surface each (Figure~\ref{fig:props-p2}).
In \method, an exact answer is possible only when the perturbation does not cross a vertical line. 
In this case, the resulting change is a set of values $\Delta \hat{y} := \{ \Delta f_i^{(r)}\}_{r=1}^{R_i}$, one for each curve in the plot: $\Delta f_i^{(r)} = f_i^{(r)}(x_i+\Delta x)-f_i^{(r)}(x_i)$. 
If a crossing occurs, an exact answer is not attainable, however, for less precise questions, such as whether the $\Delta x$ change in $x_i$ will have a positive impact on $y$, an answer is still feasible (see \textbf{P3.} below).

\textbf{P3. Global Feature Property:} Is the model globally monotonic increasing with respect to $x_i$?\\ 
\textit{Formally:} For all $\xb$ and $\Delta x>0$, is $\Delta \hat{y} = \hat y(\xb + \mathbf{e}_i \Delta x )-\hat y(\xb) > 0$?\\
The explanation is \textit{global}: it assesses the monotonicity of $x_i$ across the entire input space.
In GAM, monotonicity is easily determined by the shape of $f_i(x_i)$.
In \gamm, every row of every surface $f_{ij}$ must be checked against the curve $f_i$ (Figure~\ref{fig:props-p3}), an inspection infeasible for a human.
In \method, if no vertical lines are present, simply check whether all curves $f_i^{(r)}(x_i)$ for $r=1, \ldots, R_i$ are monotonically increasing. If vertical lines exist, simply verify that all arrows are positive.

\section{Empirical Evaluation}
\label{sec:empirical-evaluation}
\begin{figure*}[t]
  \centering
  \begin{subfigure}{0pt}\phantomsubcaption\label{fig:synth-f1}\end{subfigure}%
  \begin{subfigure}{0pt}\phantomsubcaption\label{fig:synth-f2}\end{subfigure}%
  \begin{subfigure}{0pt}\phantomsubcaption\label{fig:synth-f3}\end{subfigure}%
  \includegraphics[width=\linewidth]{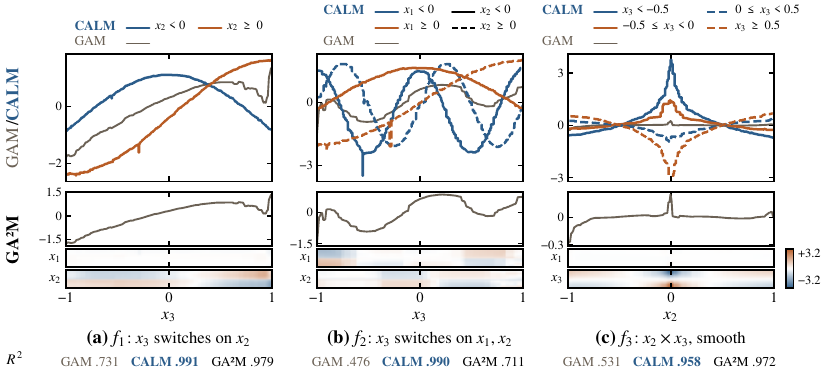}
  \caption{
    Cases~1 to 3 (one column each) of Section~\ref{sec:explanatory_example}.
    Top: \method{}'s curves, one per region, over the GAM's single curve (gray).
    Bottom: the \gamm{}'s main curve alongside its two heatmaps.
    Under each column: the mean $R^2$ over 5 folds (Table~\ref{tab:explanatory-example-r2}), with \method{} at depth~2; $f_1$'s panel is drawn at depth~1.
    }
  \label{fig:shape_functions}
\end{figure*}

We evaluate \method{} on three synthetic regression datasets and 24 public real-world tabular datasets, 10 for classification and 14 for regression.
Every experiment runs \method{} with its default configuration: XGBoost as the reference model in Step 1(a), PDP-based heterogeneity to propose the trees in Step~1(b),  EBM to fit the curves in Step~2, and the default hyperparameters of Section~\ref{subsec:model-overview} for Steps~1(b)--(c).
No hyperparameter is tuned per dataset, for both \method{} and the baselines.
Appendix~\ref{app:detailed-experiments} provides full implementation details.
Results are reported as the mean $\pm$ standard deviation over 5-fold cross-validation.

\subsection{Synthetic example}\label{sec:explanatory_example}

On three synthetic regression datasets, we compare \method{} against EBM, as the GAM, and EB$^2$M (EBM with pairwise interactions), as the \gamm{}.
In all three, the dataset consists of 1{,}000 draws of three independent features $x_i \sim \mathcal{U}(-1, 1)$; only the target $y=f_k(\xb)$ varies and each $f_k$ contains a different interaction structure:
in Case~1, a conditional interaction of two features (one curve of $x_3$ per sign of $x_2$), which both \method{} and a \gamm{} can express;
in Case~2, a conditional interaction of three features (one curve of $x_3$ per signs of $x_1$ and $x_2$), which only \method{} can express;
in Case~3, a general interaction of two features (a full $x_2$–$x_3$ term), which a \gamm{} can express and \method{} can only approximate by regions.
The targets are
\begin{align*}
f_1(\xb) &= x_1^2 + \log|x_2| + 2\sin(\tfrac{\pi}{2} x_3)\,\mathbf{1}_{x_2 \geq 0} + 2\cos(\tfrac{\pi}{2} x_3)\,\mathbf{1}_{x_2 < 0}, \\
f_2(\xb) &= x_1^2 + \log|x_2| + 2\sin(\omega x_3)\,\mathbf{1}_{x_2 \geq 0} + 2\cos(\omega x_3)\,\mathbf{1}_{x_2 < 0}, \quad \omega = \tfrac{\pi}{2}\,\mathbf{1}_{x_1 \geq 0} + 2\pi\,\mathbf{1}_{x_1 < 0}, \\
f_3(\xb) &= x_1^2 + \log|x_2| \sin(\tfrac{\pi}{2} x_3).
\end{align*}
Figure~\ref{fig:shape_functions} gives the $R^2$ of each model and Table~\ref{tab:explanatory-example-r2} (Appendix~\ref{app:synthetic}) its spread over the folds.

\textbf{Case 1.}
The GAM averages the sine and cosine modes of $x_3$ into one curve (Fig.~\ref{fig:synth-f1}, gray) and reaches $R^2 = 0.731$.
Both the \gamm{} and \method{} capture the switch on the sign of $x_2$, with near-perfect accuracy ($0.979$ and $0.991$).

\textbf{Case 2.}
The GAM merges the four modes of $x_3$ into one curve, limiting its accuracy to $R^2 = 0.476$.
The \gamm{} can only approximate the three-feature switch with its $x_1$–$x_3$ and $x_2$–$x_3$ surfaces and reaches $0.711$, while \method{} conditions $x_3$ on both signs and reaches $0.990$.

\textbf{Case 3.}
Averaged over $x_3$, the effect of $x_2$ cancels, so the GAM's curve of $x_2$ is nearly flat and its accuracy low ($R^2 = 0.531$).
The \gamm{} captures the product in its $x_2$–$x_3$ surface ($0.972$).
\method{} approximates it by splitting $x_3$ into four ranges, each with its own curve of $x_2$, and still reaches $0.958$, slightly below the \gamm{}.

\textbf{Interpretability.}
In every case, the \gamm{} spreads the effect of one feature over three views, its main curve and two heatmaps, which the reader must combine to draw conclusions (Fig.~\ref{fig:shape_functions}, bottom).
\method{} shows the same effect as a few 1D curves, each labelled by the condition that selects it (Fig.~\ref{fig:shape_functions}, top), so the questions P1--P3 of Section~\ref{subsec:interpretability-properties} are answered at a glance.

\vspace{-0.5em}
\subsection{Real-world datasets}
\label{sec:real-datasets}
On the real-world datasets, we compare \method{} with the XGBoost black box~\cite{chen2016xgboost}, two GAMs (NAM and EBM) and two \gamms{}: EB$^2$M (EBM with pairwise interactions) and NODE-GA$^2$M.
We report accuracy for classification and RMSE for regression, and count a win or a loss only when two means differ by more than the larger of their standard deviations over the folds, a draw otherwise.
Figure~\ref{fig:dumbbell} shows what each model gains over the GAM on every dataset, and Tables~\ref{tab:clf-acc} and~\ref{tab:regr-rmse} (Appendix~\ref{app:main-tables}) give every baseline with its spread, including GAMI-Net.
\begin{figure}[t]
  \centering
  \includegraphics[width=\linewidth]{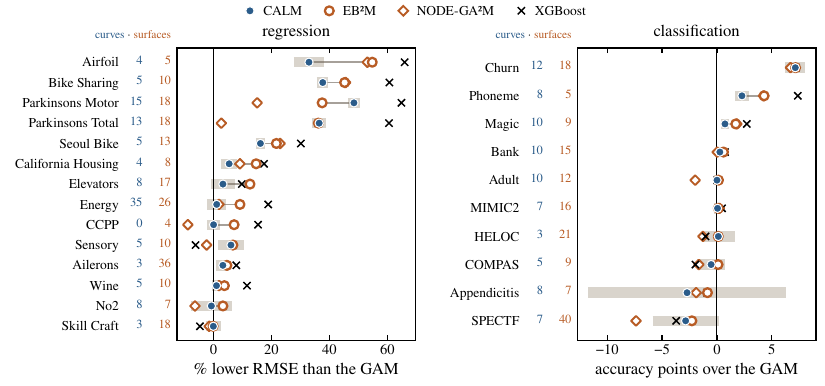}
  \caption{What each model gains over the GAM on each dataset: the black box XGBoost, the \gamms{} NODE-GA$^2$M and EB$^2$M, and \method{} (means over 5 folds, higher is better).
    The gray band behind each dot is $\pm\max(\mathrm{std}_{\text{\method{}}}, \mathrm{std}_{\text{EB}^2\text{M}})$ over the folds.
    A ring inside it is a draw, outside a win or a loss.
    \method{} ahead/level/behind:
    against the GAM, 8/6/0 in regression and 3/7/0 in classification;
    against EB$^2$M, 1/5/8 and 0/7/3 (see Tables~\ref{tab:clf-acc} and~\ref{tab:regr-rmse} for full results).
    Left of each row: the curves that \method{} adds on top of the GAM, and the pairwise surfaces that EB$^2$M fits.}
  \label{fig:dumbbell}
\end{figure}

\textbf{Classification results.}
On the 10 classification datasets (Fig.~\ref{fig:dumbbell}, right; Table~\ref{tab:clf-acc}), \method{}
(i) matches or beats each GAM on all datasets, tying with NAM on 5 and EBM on 7;
(ii) matches or beats each \gamm{} on most datasets, losing to EB$^2$M on 3 and to NODE-GA$^2$M on 2; and
(iii) matches or beats XGBoost on 6, losing on 4.

\textbf{Regression results.}
On the 14 regression datasets (Fig.~\ref{fig:dumbbell}, left; Table~\ref{tab:regr-rmse}), \method{}
(i) matches or beats each GAM on all datasets, tying with NAM on 4 and EBM on 6;
(ii) matches or beats NODE-GA$^2$M on 10, but loses to EB$^2$M on 8; and
(iii) matches or beats XGBoost on 3, losing on 11.

\textbf{Summary.}
In both tasks, \method{} improves over GAMs, closing the GAM-to-XGBoost gap by a median of 37\,\% in regression and of 31\,\% in classification.
Against the \gamms{}, it is competitive on most datasets, but it still trails: in regression, EB$^2$M is more accurate on 8 of the 14 datasets, on 7 of them by more than 5 points in Figure~\ref{fig:dumbbell}.
We attribute this gap, like the one to XGBoost, to \method{}'s stricter constraint: one-dimensional curves instead of two-dimensional surfaces.

\textbf{Complexity.}\label{sec:complexity}
We measure the complexity~\cite{molnar2020quantifying} of a \method{} in two quantities: (i) the conditional interactions of Section~\ref{subsec:model-overview}, the features $x_j$ that are part of the partition tree $T_i$ of $x_i$ (Appendix~\ref{app:interactions}) and (ii) the number of curves it adds to the GAM, which we compare with a \gamm{}'s pairs and surfaces.
For (i), \method{} uses 6.6 conditional interactions on average in regression and 6.8 in classification (Tables~\ref{tab:classification_interactions} and~\ref{tab:regression_interactions}); for (ii), it adds 8.2 and 7.9 curves (Fig.~\ref{fig:dumbbell}, left of each row).
EB$^2$M fits 14.3 and 15.2 pairs, one surface each.
These numbers come from \method{}'s default configuration, which puts no cap on the number of interactions.
\method{} still finds a sparse structure, because a tree grows only when a split reduces heterogeneity enough (Step~1(b)), and the selection keeps only the trees that add fidelity to the reference model, i.e., that increase $R^2(\mathcal{S})$ (Step~1(c)).
This is a key advantage of \method{}: it adds an interaction only when the data asks for it.

\textbf{Sensitivity to hyperparameters.}
We vary $d_{\max}$ and $\gamma$ and find that only the depth affects accuracy and complexity substantially.
Over the datasets where XGBoost beats the GAM, i.e., where there is room for improvement (12 of 14 in regression, 5 of 10 in classification), the median share of that gap that \method{} closes grows from 17\,\% at $d_{\max}=1$ to 37\,\% at $d_{\max}=2$ and 49\,\% at $d_{\max}=3$ in regression (17, 31 and 41\,\% in classification), while the median number of conditional interactions roughly doubles per level (1.9, 4.5 and 9.6 in regression).
On Bike Sharing, deeper trees keep paying: $d_{\max}=6$ closes 85\,\% of the gap, against 62\,\% at the default depth, with 9.8 conditional interactions instead of 4.6 (Appendix~\ref{app:depth-ladder}).
In general, depth is a knob to trade interpretability for accuracy, and the default $d_{\max}=2$ is a good compromise.
Halving or doubling $\gamma$ changes the accuracy beyond one standard deviation on at most two of the 24 datasets, so its default is a safe choice.

\textbf{Runtime.}
On a laptop CPU (10 threads), \method{} fits 19 of the 24 datasets in under a minute (median 21\,s, at most 2\,min, on Energy), against a median of 4\,s for EBM on the same runs.
Most of this time is spent proposing the trees (Step~1(b), $\mathcal{O}(d_{\max}\, d^2 N \log N)$) and fitting the shape functions (Step~2).
NODE-GA$^2$M, timed earlier on other hardware, takes 1 to 2 minutes per fit (Tables~\ref{tab:clf-time-def}--\ref{tab:regr-time-def}).

\subsection{Stability}\label{sec:stability}
\begin{wraptable}{r}{0.47\textwidth}
\vspace{-1.2em}
\centering
\small
\setlength{\tabcolsep}{4pt}
\begin{tabular}{lcccc}
\toprule
& \multicolumn{2}{c}{folds} & \multicolumn{2}{c}{bootstrap} \\
\cmidrule(lr){2-3}\cmidrule(lr){4-5}
agreement on the & 5/5 & $\geq$4/5 & 5/5 & $\geq$4/5 \\
\midrule
interaction verdict & 77\,\% & 90\,\% & 74\,\% & 89\,\% \\
root conditioner & 74\,\% & 87\,\% & 71\,\% & 86\,\% \\
full structure & 71\,\% & 81\,\% & 68\,\% & 80\,\% \\
\bottomrule
\end{tabular}
\caption{Share of the 376 features on which the refits of \method{} agree.}
\label{tab:stability}
\vspace{-1em}
\end{wraptable}
We assess the stability of \method{}: does refitting it return similar trees?
This matters when \method{} is used not only to predict with an interpretable-by-design model, but also to learn about the process that generated the data~\cite{philipp2018measuring, freiesleben2024scientific}.
We therefore refit \method{} on the 5 cross-validation folds (each on 80\,\% of the rows) and on 20 bootstrap resamples, and ask three questions of increasing strictness: (i) is the feature split at all (the \emph{interaction verdict}), (ii) on which feature is its root split (the \emph{root conditioner}), and (iii) is its whole tree the same (the root and the depth-2 conditioners, the \emph{full structure}); ``not split'' counts as an answer.
Each question is asked separately of each of the 376 features of the 24 datasets.
Table~\ref{tab:stability} gives the share of features on which all 5 refits agree, or at least 4; the bootstrap shares are averaged over 4 disjoint groups of 5 resamples.
On whether a feature is split, 77\,\% of the features agree across all five folds and 90\,\% across four; agreement falls only a few points as the question becomes stricter, to 71\,\% and 81\,\% for the full structure, and the bootstrap, which perturbs the data more, lowers each share by one to three points.
Where a rule recurs, its threshold barely moves: over the 88 rules found in at least 3 of the 5 folds, the median spread of the threshold is 0.08 standard deviations of the feature, and 73\,\% stay within 0.2.

\subsection{User study}\label{sec:user-study}
In an online study, 25 participants, from novices to experts in machine learning (13 self-reported advanced or expert), answered questions on the effect of the hour in a Bike Sharing model, once from \method{}'s conditional curves and once from a \gamm{}'s main effect and heatmaps~\cite{hegselmann2020evaluation}: how the prediction changes when the hour moves~\cite{hase2020evaluating}.
Participants answered 54\,\% of the questions correctly with \method{} and 34\,\% with the \gamm{} ($p<0.05$), and the gap is widest on the question of how long the hour takes to lower the prediction (40\,\% vs.\ 12\,\%).
The advantage holds among the 13 advanced or expert participants, for whom heatmaps are standard tooling (65.4\,\% vs.\ 30.8\,\%, Wilcoxon $p=0.031$), and accuracy with the \gamm{} does not rise with expertise (Spearman $\rho=-0.29$, $p=0.14$); Appendix~\ref{app:user-study} gives the design, the questions and the grading. Order and dataset were not counterbalanced.

\vspace{-0.5em}
\section{Conclusion}\label{sec:conclusion}
\label{sec:limitations}
We introduced \method{}, an interpretable-by-design model that narrows the gap between the interpretability of GAMs and the accuracy of \gamms{}: it models interactions as \emph{conditional feature effects}, one univariate curve per feature and region of its interacting features.
On 24 datasets, \method{} never loses to a GAM, adds only about eight curves to it and closes about a third of the gap to XGBoost; in a user study, it is read more accurately than a \gamm{}.
Its main limitation is that interpretability may degrade as interactions grow: although \method{} uses only 6.6 conditional interactions on average in regression and 6.8 in classification, a feature that depends on many others gets many curves and vertical dotted lines, which make it harder to read~\cite{abdul2020cogam}.
Future work includes replacing the distillation pipeline with end-to-end training, e.g., a neural network under the same structural constraints, and studying the identifiability of \method{} and its connection to the functional ANOVA decomposition~\cite{hooker2007generalized, lengerich2020purifying, hiabu2023unifying}.
Python package: \url{https://pypi.org/project/calm-additive}.

\begin{ack}
The work leading to these results has been funded by the European Union under Grant Agreement No.~101289094 (project AIRIS). Views and opinions expressed are however those of the authors and do not necessarily reflect those of the European Union or the granting authority (HaDEA). Neither the European Union nor the granting authority can be held responsible for them.

The research leading to this work has received funding from the European Union's Horizon Europe research and innovation program under Grant Agreement No.~101135826 (\href{https://aidapt.eu}{aidapt.eu}).

This work was also supported by the EYWA project, winner of the EIC Horizon Prize on Early Warning for Epidemics, and the CONOPS project, funded by the European Union through the Destination Earth initiative and implemented by ESA (contract No.~4000143571/24/I-NS).

The authors declare no competing interests.
\end{ack}

\bibliography{bibliography}
\bibliographystyle{plain}

\clearpage
\appendix
\section{Conditionally Additive Local Models (CALMs)}
\label{app:learning-rams}

Algorithm~\ref{alg:calm-training} fits a CALM (Eq.~\eqref{eq:CALM}) to a dataset $\mathcal{D} = \{\xb^{(k)}, y^{(k)}\}_{k=1}^N$ in two steps. Step~1 finds the regions: (a) train a reference model \(f_{\mathtt{ref}}\) on $\mathcal{D}$, (b) propose a partition tree $T_i$ per feature from a heterogeneity measure $H_i$, and (c) select the trees that enter the model. Step~2 fits the shape functions \( \{ f_i^{(r)} \}_{r=1}^{R_i} \) of each feature.
We detail each step below.

\subsection{Step 1(a): Train a reference model \(f_{\mathtt{ref}}\)}

Any sufficiently accurate model can serve as \(f_{\mathtt{ref}}: \mathbb{R}^d \rightarrow \mathbb{R}\), e.g., gradient-boosted trees (XGBoost~\cite{chen2016xgboost}, our default), neural networks, random forests~\cite{breiman2001random}, or TabPFN~\cite{hollmann2022tabpfn}.

\subsection{Step 1(b): Propose a partition tree $T_i$ per feature}

We first define the heterogeneity measure $H_i$ (Section~\ref{subsec:heterogeneity-measures}) and relate it to CALM's error in approximating $f_{\texttt{ref}}$. Motivated by this result, we fit a candidate partition tree \(T_i\) for each feature \(i\) whose heterogeneity is at or above the median over all features (Section~\ref{subsec:partition-tree}).

\subsubsection{Heterogeneity $H_i$}
\label{subsec:heterogeneity-measures}

Interaction-related heterogeneity $H_i$~\cite{herbinger2024GADGET} quantifies the extent to which a feature $x_i$ interacts with all other features \(x_j\), where \(j \in \{1, \ldots, d\} \setminus \{i\}\), like the H-statistic~\cite{friedman2008predictive}.
The computation of $H_i$ is based on the variance of local effects: $h(x_i , \xbi^{(j)})$.

\textbf{Local effect.}
The local effect $h(x_i , \xbi^{(j)})$ is the change in the output of $f_{\mathtt{ref}}$ when the $i$-th feature of the $j$-th background observation is set to $x_i$~\cite{goldstein2015peeking}.
Its form depends on the feature-effect method, e.g., Partial Dependence Plot (PDP), SHAP-Dependence Plot (SHAP-DP) or Robust and Heterogeneity-aware ALE (RHALE).
We use PDP by default:

\begin{equation}
h\bigl(x_i,\,\mathbf{x}_{-i}^{(j)}\bigr) = 
f_{\mathtt{ref}}\bigl(x_i,\,\mathbf{x}_{-i}^{(j)}\bigr)
- c(\mathbf{x}^{(j)}_{-i})
\end{equation}

where $c(\mathbf{x}^{(j)}_{-i})= \mathbb{E}_{X_i} \left [ f_{\mathtt{ref}}\bigl(X_i,\,\mathbf{x}_{-i}^{(j)}\bigr) \right ]$ is a centering constant~\cite{herbinger2024GADGET}.
For details on alternative approaches, we refer the reader to~\cite{herbinger2024GADGET, herbinger2022REPID}.

\textbf{Correlated features.}
While PDP can suffer from extrapolation bias under feature correlations~\cite{molnar2022general, molnar2024model}, here it is used only for region detection, not for prediction: even biased local effects yield elevated heterogeneity when genuine interactions are present. When correlations are strong, RHALE \citep{gkolemis2023rhale} is a drop-in alternative. A related caveat is that an interaction with one feature can also raise heterogeneity along its correlated proxies, so CALM may split on either; both choices recover the interaction structure, though the conditioning labels may not be uniquely identifiable under strong collinearity, a limitation shared by partition-based interaction detectors such as REPID~\cite{herbinger2022REPID}.

\textbf{Point-wise heterogeneity $H_i(x_i)$.}
The pointwise heterogeneity (heterogeneity at a specific point $x_i$)
is then defined as 
\begin{equation}
H_i(x_i) = \mathbb{E}_{\mathbf{X}_{-i}}\left[\left(h(x_i, \mathbf{X}_{-i}) - \mu_i(x_i)\right)^2\right] = \text{Var}_{\mathbf{X}_{-i}}\left[h(x_i, \mathbf{X}_{-i})\right]
\end{equation}
where $\mu_i(x_i) = \mathbb{E}_{\mathbf{X}_{-i}}[h(x_i, \mathbf{X}_{-i})]$ is the mean local effect. Inside a region, this heterogeneity takes the form of Eq. \eqref{eq:pointwiseh}.

\textbf{Feature-level heterogeneity $H_i$.}
Feature-level heterogeneity (heterogeneity of the feature) shows the overall strength of interactions of the feature with all other features, and is defined as the root of the expected pointwise heterogeneity, i.e., the root energy of the 1d curve of pointwise heterogeneity:
\begin{equation}
H_i = \sqrt{\mathbb{E}_{X_i}\left[H_i(X_i)\right]} = \sqrt{\mathbb{E}_{X_i}\left[\text{Var}_{\mathbf{X}_{-i}}[h(X_i, \mathbf{X}_{-i})]\right]}
\end{equation}
Inside a region this definition takes the form of Eq. \eqref{eq:heterogeneity}.
The square root places $H_i$ in the units of the output. In the proposition and proof below, the natural object is the squared quantity $H_i^2$, the expected variance of the local effects, and we write it explicitly.

\subsubsection{Heterogeneity and CALM approximation error}

Inside a region, the best single curve for $x_i$ is the average of its local effects, and what this curve misses is exactly the heterogeneity left in the region.
The proposition below sums these misses over regions and features into a bound on CALM's error in approximating $f_{\texttt{ref}}$, so a split that lowers the heterogeneity lowers the bound, i.e., raises a floor on CALM's $R^2$ with respect to $f_{\texttt{ref}}$; Step~1(b) lowers a computable proxy of this heterogeneity, the PDP heterogeneity of $f_{\texttt{ref}}$.
For example, in Bike Sharing the effect of the hour differs between working days and days off, so a single curve for the hour misses much of it and its heterogeneity is high; split on working day, each curve fits its own rows and the heterogeneity drops.

\begin{proposition}
    Let $\mathcal{E} = \mathbb{E}_{\mathbf{X}}\left[\left(f_{\texttt{ref}}(\mathbf{X}) - f_{\texttt{CALM}}(\mathbf{X})\right)^2\right]$ be the mean squared error of the CALM approximation of $f_{\texttt{ref}}$. Assume that $f_{\texttt{ref}}: \mathcal{X} \to \mathbb{R}$ admits a multivariate regionally additive approximation
\begin{equation*}
f_{\texttt{ref}}(\mathbf{x}) \approx \beta_0 + \sum_{i = 1}^d\psi_i^{r_i(\mathbf{x}_{-i})}(\mathbf{x})
\end{equation*}
such that each $\psi_i^{r_i(\mathbf{x}_{-i})}$ captures the contribution of the $i$-th feature to the function's output (including interactions). Also, assume that $X_i \perp \mathbf{X}_{-i}$ within each region (Remark~\ref{rem:independence}) and that there is no error in fitting of the univariate CALM shape functions. For clarity of exposition, the proof treats this approximation as an exact equality; the bound \eqref{eq:error_bound} holds up to the residual approximation variance. Then,
\begin{equation}
\min_{f_{\texttt{CALM}}} \mathcal{E} \leq d\sum_{i=1}^{d}\left(\sum_{r=1}^{R_i}P_{\mathcal{R}_i^{(r)}}\left(H^{\mathcal{R}_i^{(r)}}_i\right)^2\right)
\label{eq:error_bound}
\end{equation}
where region $\mathcal{R}_i^{(r)}$ is the $r$-th region of the $i$-th feature, $P_{\mathcal{R}_i^{(r)}}$ is the probability mass of the region in the data and $H^{\mathcal{R}_i^{(r)}}_i$ is the heterogeneity of $\psi_i^{\mathcal{R}_i^{(r)}}(\mathbf{x})$ as in Eq.~\eqref{eq:heterogeneity}, with $\psi_i$ in place of the local effect $h$.
\label{prop:step2}
\end{proposition}
With zero heterogeneity in every region the bound is zero, the per-feature counterpart of GADGET's result that regions without interaction-related heterogeneity make the regional effects additive~\cite[Thm.~3 and Eq.~8]{herbinger2024GADGET}.

\begin{proof}
For the approximation error we have
\begin{align*}
\mathcal{E} &= \mathbb{E}_{\mathbf{X}}\left[\left(f_{\texttt{ref}}(\mathbf{X}) - f_{\texttt{CALM}}(\mathbf{X})\right)^2\right] = \mathbb{E}_{\mathbf{X}}\left[\left(\left(\beta_0+ \sum_{i = 1}^d\psi_i^{r_i(\mathbf{x}_{-i})}(\mathbf{x})\right) - \left(\beta_0 + \sum_{i=1}^{d}f_i^{r_i(\mathbf{x}_{-i})}(x_i)\right)\right)^2\right]\nonumber \\
&=\mathbb{E}_{\mathbf{X}}\left[\left(\sum_{i = 1}^d\left(\psi_i^{r_i(\mathbf{x}_{-i})}(\mathbf{x}) - f_i^{r_i(\mathbf{x}_{-i})}(x_i)\right)\right)^2\right]
\leq d\sum_{i=1} ^{d}\mathbb{E}_{\mathbf{X}}\left[\left(\psi_i^{r_i(\mathbf{x}_{-i})}(\mathbf{x}) - f_i^{r_i(\mathbf{x}_{-i})}(x_i)\right)^2\right]
\end{align*}
where the last step results from the Jensen's inequality. We therefore have
\begin{equation}
\mathcal{E} \leq d\sum_{i=1}^{d}\mathcal{E}^{r_i(\mathbf{x}_{-i})}_i
 \label{eq:error_approx}
\end{equation}
Next, we show that the minimum error achieved for each region is equal to the heterogeneity of the local effects. For any measurable function $g_i(x_i)$ and region $\mathcal{R}$, decompose, using $\mu^{\mathcal{R}}_i(x_i) = \mathbb{E}_{\mathbf{X}_{-i}\mid \mathbf{X}_{-i} \in \mathcal{R}}\left[\psi^{\mathcal{R}}_i(x_i, \mathbf{X}_{-i})\right]$:
\begin{equation*}
\psi_i(X_i, \mathbf{X}_{-i}) - g_i(X_i) = \left(\psi_i(X_i, \mathbf{X}_{-i}) - \mu^{\mathcal{R}}_i(X_i)\right) + \left(\mu^{\mathcal{R}}_i(X_i) - g_i(X_i)\right)
\end{equation*}

Taking expectations inside $\mathcal{R}$ and squaring:
\begin{align*}
&\mathbb{E}_{X_i, \mathbf{X}_{-i} \mid \mathbf{X}_{-i}\in \mathcal{R}}\left[\left(\psi_i^{\mathcal{R}}(X_i, \mathbf{X}_{-i}) - g_i(X_i)\right)^2\right] \\
&= \mathbb{E}_{X_i, \mathbf{X}_{-i} \mid \mathbf{X}_{-i}\in \mathcal{R}}\left[\left(\psi_i^{\mathcal{R}}(X_i, \mathbf{X}_{-i}) - \mu^{\mathcal{R}}_i(X_i)\right)^2\right] \\
&\quad + \mathbb{E}_{X_i}\left[\left(\mu^{\mathcal{R}}_i(X_i) - g_i(X_i)\right)^2\right] \\
&\quad + 2\mathbb{E}_{X_i}\left[\mathbb{E}_{\mathbf{X}_{-i}\mid \mathbf{X}_{-i} \in \mathcal{R}}\left[\psi_i^{\mathcal{R}}(X_i, \mathbf{X}_{-i}) - \mu^{\mathcal{R}}_i(X_i)\right]\left(\mu^{\mathcal{R}}_i(X_i) - g_i(X_i)\right)\right]
\end{align*}

The cross-term equals zero because:
\begin{equation*}
\mathbb{E}_{\mathbf{X}_{-i}\mid \mathbf{X}_{-i} \in \mathcal{R}}\left[\psi_i^{\mathcal{R}}(x_i, \mathbf{X}_{-i}) - \mu^{\mathcal{R}}_i(x_i)\right] = \mu^{\mathcal{R}}_i(x_i) - \mu^{\mathcal{R}}_i(x_i) = 0
\end{equation*}
by definition of $\mu^{\mathcal{R}}_i$. Therefore:
\begin{equation*}
\mathbb{E}_{X_i, \mathbf{X}_{-i}\mid \mathbf{X}_{-i} \in \mathcal{R}}\left[\left(\psi_i^{\mathcal{R}}(X_i, \mathbf{X}_{-i}) - g_i(X_i)\right)^2\right] = \mathcal{E}^{\mathcal{R}}_i + \mathbb{E}_{X_i}\left[\left(\mu^{\mathcal{R}}_i(X_i) - g_i(X_i)\right)^2\right] \geq \mathcal{E}^{\mathcal{R}}_i
\end{equation*}
with equality if and only if $g_i = \mu^{\mathcal{R}}_i$ almost surely. This indicates that within each region, $\mu_i^{\mathcal{R}}$ minimizes the approximation error.

Moreover, from the definition of pointwise heterogeneity:
\begin{equation*}
H^{\mathcal{R}}_i(x_i) := \mathbb{E}_{\mathbf{X}_{-i}\mid \mathbf{X}_{-i} \in \mathcal{R}}\left[\left(\psi_i^{\mathcal{R}}(x_i, \mathbf{X}_{-i}) - \mu^{\mathcal{R}}_i(x_i)\right)^2\right]
\end{equation*}

Taking the expectation over $X_i$:
\begin{equation*}
\left(H^{\mathcal{R}}_i\right)^2 = \mathbb{E}_{X_i}\left[H^{\mathcal{R}}_i(X_i)\right] = \mathbb{E}_{X_i}\left[\mathbb{E}_{\mathbf{X}_{-i}\mid \mathbf{X}_{-i} \in \mathcal{R}}\left[\left(\psi_i^{\mathcal{R}}(X_i, \mathbf{X}_{-i}) - \mu^{\mathcal{R}}_i(X_i)\right)^2\right]\right]
\end{equation*}

which gives
\begin{equation*}
\left(H^{\mathcal{R}}_i\right)^2 = \mathbb{E}_{X_i, \mathbf{X}_{-i}|\mathbf{X}_{-i} \in \mathcal{R}}\left[\left(\psi_i^{\mathcal{R}}(X_i, \mathbf{X}_{-i}) - \mu^{\mathcal{R}}_i(X_i)\right)^2\right] = \min\mathcal{E}_i^{\mathcal{R}}
\end{equation*}
i.e., the squared heterogeneity is the minimum approximation error.

Since the squared heterogeneity is the minimum error for each region, we can estimate the feature-level error across regions as
\begin{equation*}
\min{\mathcal{E}_i} = \sum_{r=1}^{R_i}P_{\mathcal{R}_i^{(r)}}\left(H^{\mathcal{R}_i^{(r)}}_i\right)^2
\end{equation*}
where $P_{\mathcal{R}_i^{(r)}}$ is the probability mass of region $\mathcal{R}_i^{(r)}$.

Using this result with Eq. \eqref{eq:error_approx}, we have

\begin{equation*}
\min_{f_{\texttt{CALM}}} \mathcal{E} \leq d\sum_{i=1}^{d}\left(\sum_{r=1}^{R_i}P_{\mathcal{R}_i^{(r)}}\left(H^{\mathcal{R}_i^{(r)}}_i\right)^2\right)
\end{equation*}

\end{proof}

\begin{remark}\label{rem:independence}
The switch from the conditional expectation $\mathbb{E}_{X_i \mid \mathbf{X}_{-i} \in \mathcal{R}}$ to the marginal $\mathbb{E}_{X_i}$ in the second term of the decomposition above implicitly assumes $X_i \perp \mathbf{X}_{-i} \mid \mathbf{X}_{-i} \in \mathcal{R}$ within each region. This is the standard assumption underlying PDP-style decompositions and is shared by all feature-effect methods that use the marginal distribution of $X_i$.
\end{remark}

The bound links heterogeneity to the error of a single curve inside a region, which motivates splitting on it. The additive approximation of $f_{\texttt{ref}}$ is the usual assumption of feature-effect methods, and the local effects and their heterogeneity are used only to find the regions, not to estimate $f_{\texttt{ref}}$.

\subsubsection{Heterogeneity estimation}

For heterogeneity estimation, denote with $\mathcal{I} \subseteq \{1, \ldots, N\} $ the index set of active background instances considered in the heterogeneity calculation i.e., those that reside in region $\mathcal{R}^{(r_i(\mathbf{x}_{-i}))}_i$ of the $i$-th instance.
Then, the point-wise heterogeneity at feature value $x_i$ over $\mathcal{I}$ can be estimated by

\begin{equation}\label{eq:point-wise-heterogeneity}
\hat{H}_i^{\mathcal{I}}(x_i) = \frac{1}{|\mathcal{I}|}\sum_{j \in \mathcal{I}} \left ( h(x_i , \xbi^{(j)}) - \mu(x_i) \right )^2 \text{, where } \mu(x_i)=\frac{1}{|\mathcal{I}|} \sum_{j \in \mathcal{I}} h(x_i , \xbi^{(j)})
\end{equation}

Eq.\eqref{eq:point-wise-heterogeneity} quantifies the strength of the interactions of the $i$-th feature, at position $x_i$, as 
the variance of the local effects at $x_i$.
To obtain a global measure of these interactions, we average the pointwise heterogeneity over a grid of $M$ values $\{ {\tilde{x}_i^{(m)}} \}_{m=1}^M$ sampled from the domain of $x_i$:
\begin{equation}
\label{eq:heter-definition}
\hat{H}_i^{\mathcal{I}} = \sqrt{\frac{1}{M} \sum_{m=1}^M \hat{H}_i^{\mathcal{I}}(\tilde{x}_i^{(m)})}.
\end{equation}
The square root returns the heterogeneity to the units of the output, so that $\hat{H}_i^{\mathcal{I}}$ is the typical deviation of a local effect from the average effect over $\mathcal{I}$; the relative drop of Eq.~\eqref{eq:heter-drop} below is computed on this scale.

\subsubsection{Computing the partition tree $T_i$ for each feature.}
\label{subsec:partition-tree}

\begin{algorithm}[tb]
\caption{Fitting a Partition Tree $T_i$ for Feature $x_i$}
\label{alg:fit-tree}
\begin{algorithmic}[1]
   \STATE {\bfseries Input:} $f_{\text{ref}}$, $\mathcal{D} = \{(\mathbf{x}^{(k)}, y^{(k)})\}_{k=1}^N$; {\bfseries Parameters:} max depth $d_{\max}$, threshold $\epsilon$, grid size $M=20$
   \STATE {\bfseries Output:} Binary tree $T_i$
   \STATE
   \STATE {\bfseries function} \textsc{BuildTree}(node $\nu$, depth $\ell$, active indices $\mathcal{I}$)
   \IF{$\ell = d_{\max}$}
      \STATE \textbf{return} $\nu$ \COMMENT{Stop splitting}
   \ENDIF
   \STATE Compute $H_i^{\mathcal{I}}$ using Eq.\eqref{eq:heter-definition}
   \STATE Initialize: $\Delta H_{\max} \leftarrow 0$
   \FOR{each feature $j \neq i$}
      \STATE Determine candidate thresholds $\mathcal{T}_j$ \COMMENT{Default: Unique values if $x_j$ categorical, else $M$-grid (default: $M=20$)}
      \FOR{each $\tau \in \mathcal{T}_j$}
        \IF{$x_j$ is numerical}
            \STATE $\mathcal{I}_L \leftarrow \{k \in \mathcal{I} : x_j^{(k)} < \tau\}$, $\mathcal{I}_R \leftarrow \mathcal{I} \setminus \mathcal{I}_L$
        \ELSE[$x_j$ is categorical]
            \STATE $\mathcal{I}_L \leftarrow \{k \in \mathcal{I} : x_j^{(k)} = \tau\}$, $\mathcal{I}_R \leftarrow \mathcal{I} \setminus \mathcal{I}_L$
        \ENDIF
         \STATE Compute $\Delta H_i$ using Eq.~\eqref{eq:heter-drop}
         \IF{$\Delta H_i > \Delta H_{\max}$}
            \STATE $\Delta H_{\max} \leftarrow \Delta H_i$, Store $(j, \tau) $ as optimal split
         \ENDIF
      \ENDFOR
   \ENDFOR
   \IF{$\Delta H_{\max} > \epsilon$}
      \STATE Split node $\nu$ into $\nu_L, \nu_R$ using optimal $(j, \tau)$
      \STATE $\nu_L \leftarrow$ \textsc{BuildTree}($\nu_L, \ell + 1, \mathcal{I}_L$)
      \STATE $\nu_R \leftarrow$ \textsc{BuildTree}($\nu_R, \ell + 1, \mathcal{I}_R$)
   \ENDIF
   \STATE \textbf{return} $\nu$
   \STATE {\bfseries end function}
\end{algorithmic}
\end{algorithm}

For each feature \(x_i\) whose heterogeneity \(\hat{H}_i\) over the full training data is at or above the median over all features, we construct a binary tree \(T_i\) of maximum depth \(d_{\max}\); the remaining features keep a single global shape function.
At each internal node of the tree, we evaluate candidate binary splits on every other feature \(x_j\), \(j \neq i\), as in CART~\cite{breiman1984cart}.
The best split is the one with the largest heterogeneity drop: a candidate split partitions the current instances $\mathcal{I}$ into $\mathcal{I}_L$ and $\mathcal{I}_R$, and
\begin{equation}
\label{eq:heter-drop}
\Delta H_i = \frac{H_i^{\mathcal{I}} - \left ( \frac{|\mathcal{I}_L|}{|\mathcal{I}|}H_i^{\mathcal{I}_L} + \frac{|\mathcal{I}_R|}{|\mathcal{I}|}H_i^{\mathcal{I}_R}  \right )}{H_i^{\mathcal{I}}}
\end{equation}

where \(H_i^{\mathcal{I}_L}\) and \(H_i^{\mathcal{I}_R}\) are the heterogeneities of \(x_i\) on the two children, weighted by their share of the instances.
The drop is relative, so the threshold $\epsilon$ reads as a fraction: we use $\epsilon = 0.05$, i.e., a split must reduce the heterogeneity of the node by at least $5\%$ on the standard-deviation scale of Eq.~\eqref{eq:heter-definition}.
A split is also rejected if either child would contain fewer than $10$ training instances.
For numerical conditioning features we evaluate $M = 20$ equally spaced candidate thresholds; for categorical conditioning features we evaluate one-vs-rest splits over the observed levels.

The procedure for fitting the tree $T_i$ is given in Algorithm \ref{alg:fit-tree}.

\subsection{Step 1(c): Select the trees to keep}
\label{subsec:selection}

The trees of Section~\ref{subsec:partition-tree} are proposed one feature at a time and are therefore candidates: an interaction between $x_i$ and $x_j$ raises the heterogeneity of both features~\cite[Thm.~2]{herbinger2024GADGET}, so both trees typically split on each other, although one of them suffices to express the interaction.
Applying every candidate double-counts such interactions in the regional curves; the resulting model is larger and can be \emph{less} faithful to $f_{\texttt{ref}}$ than a model that keeps only one of the two trees.
We therefore select the subset $\mathcal{S}$ of trees that enter the model by greedy forward selection on an additive surrogate of $f_{\texttt{ref}}$.

\textbf{The surrogate.}
For a candidate set $\mathcal{S}$, every feature contributes the average local effect of $f_{\texttt{ref}}$ computed in Section~\ref{subsec:heterogeneity-measures}: the global curve $\mu_i(x_i)$ if $i \notin \mathcal{S}$, and the regional curve $\mu_i^{\mathcal{R}_i^{(r)}}(x_i)$ of the region $\mathcal{R}_i^{(r)}$ that $\xbi$ falls in if $i \in \mathcal{S}$.
Each curve is centered over the instances it applies to, so that the curves carry only shape.
What the curves leave over is absorbed by a piecewise-constant offset $c_{\mathcal{S}}(\xb)$: a global intercept plus one coefficient per region of every tree in $\mathcal{S}$ (one region per tree is dropped to keep the design full rank), fitted jointly by least squares to the residual of the curves on the training data.
With a single tree the offsets reduce to the per-region means of the residual; with several trees the joint fit assigns a shared mean shift once, rather than to each overlapping region.
The surrogate is $g_{\mathcal{S}}(\xb) = c_{\mathcal{S}}(\xb) + \sum_{i=1}^d \mu_i^{\mathcal{R}_i(\xbi)}(x_i)$, in the notation of Eq.~\eqref{eq:surrogate-r2}, and its fidelity~\cite{tan2023considerations, szepannek2023how} is
\begin{equation}
\label{eq:surrogate-r2-app}
R^2(\mathcal{S}) = 1 - \frac{\frac{1}{N}\sum_{k=1}^N \left(f_{\texttt{ref}}(\xb^{(k)}) - g_{\mathcal{S}}(\xb^{(k)})\right)^2}{\mathrm{Var}_k\left[f_{\texttt{ref}}(\xb^{(k)})\right]},
\end{equation}
which may be negative and is not clamped.
$R^2(\emptyset)$ is the fidelity of the GAM surrogate, and $R^2(\mathcal{S}) - R^2(\emptyset)$ is what the trees in $\mathcal{S}$ add.
Scoring a candidate set requires no further evaluations of $f_{\texttt{ref}}$: the curves are read from the local effects already computed for the heterogeneity.

\textbf{Greedy forward selection.}
Algorithm~\ref{alg:select-trees} starts from $\mathcal{S} = \emptyset$ and, in each round, scores every remaining candidate on top of the trees already accepted.
The candidate with the largest $R^2(\mathcal{S} \cup \{i\})$ is accepted if its gain over $R^2(\mathcal{S})$ is at least $\gamma$; otherwise the procedure stops.
It also stops when the next acceptance would exceed the cap $K$ on the number of conditional interactions, counted as one per (feature, conditioning feature) pair appearing in the accepted trees; a capped selection is therefore a prefix of the uncapped one.
Because each gain is measured on top of the current set, the gains of the accepted trees sum exactly to $R^2(\mathcal{S}) - R^2(\emptyset)$, and a candidate whose interaction is already expressed by an accepted tree shows a gain close to zero.
Rejected candidates are recorded with the reason: \emph{redundant} if their gain is close to zero, \emph{below threshold} if it is real but smaller than $\gamma$, or the cap that stopped the sequence.
In all experiments, $\gamma = 0.01$, i.e., a tree must explain at least one additional percentage point of the variance of $f_{\texttt{ref}}$, and $K$ is unconstrained.

\begin{algorithm}[tb]
\caption{Selecting the trees to keep}
\label{alg:select-trees}
\begin{algorithmic}[1]
   \STATE {\bfseries Input:} candidate trees $\{T_i\}_{i \in \mathcal{C}}$, local effects of $f_{\texttt{ref}}$ on $\mathcal{D}$; {\bfseries Parameters:} gain threshold $\gamma$, cap $K$ (optional)
   \STATE {\bfseries Output:} the set $\mathcal{S} \subseteq \mathcal{C}$ of trees that enter the model
   \STATE $\mathcal{S} \leftarrow \emptyset$, $\rho \leftarrow R^2(\emptyset)$ \COMMENT{the GAM surrogate, Eq.~\eqref{eq:surrogate-r2-app}}
   \WHILE{$\mathcal{C} \setminus \mathcal{S} \neq \emptyset$}
      \STATE $i^\star \leftarrow \arg\max_{i \in \mathcal{C} \setminus \mathcal{S}} R^2(\mathcal{S} \cup \{i\})$ \COMMENT{score every candidate on top of $\mathcal{S}$}
      \IF{$R^2(\mathcal{S} \cup \{i^\star\}) - \rho < \gamma$}
         \STATE \textbf{break} \COMMENT{no candidate clears the gain threshold}
      \ENDIF
      \IF{$K$ is set and the conditional interactions of $\mathcal{S} \cup \{i^\star\}$ exceed $K$}
         \STATE \textbf{break} \COMMENT{the cap stops the sequence}
      \ENDIF
      \STATE $\mathcal{S} \leftarrow \mathcal{S} \cup \{i^\star\}$, $\rho \leftarrow R^2(\mathcal{S})$
   \ENDWHILE
   \STATE \textbf{return} $\mathcal{S}$
\end{algorithmic}
\end{algorithm}

The output of the selection is a sequence of models, from the GAM surrogate to the final CALM, with the gain of every accepted tree and the reason for every rejected one.
We report it alongside each fitted model, since it states in one table which interactions the model carries and what each of them contributes to the fidelity to $f_{\texttt{ref}}$.

\subsection{Step 2: Fit the shape functions}

A GAM is fitted by round-robin gradient boosting \cite{friedman2001greedy}: each iteration updates one shape function \(f_i\) on the current residuals of the whole dataset (Algorithm~\ref{alg:gb-gam}, as in~\cite{lou2012intelligible}).

\begin{algorithm}[tb]
\caption{Gradient Boosting for GAM}
\label{alg:gb-gam}
\begin{algorithmic}[1]
   \STATE Initialize $f_i \leftarrow 0$ for all $i = 1, \dots, d$
   \FOR{$m = 1$ {\bfseries to} $M$}
      \FOR{$i = 1$ {\bfseries to} $d$}
         \STATE $\mathcal{E} \leftarrow \{(x^{(j)}_i, y^{(j)} - f_{\text{GAM}}(\mathbf{x}^{(j)}))\}_{j=1}^N$ \COMMENT{Compute residuals}
         \STATE Learn shape function $S : x_i \rightarrow \mathbb{R}$ using $\mathcal{E}$ \COMMENT{Fit shape function}
         \STATE $f_i \leftarrow f_i + \eta S$ \COMMENT{Update with learning rate $\eta$}
      \ENDFOR
   \ENDFOR
\end{algorithmic}
\end{algorithm}

For a CALM, each update of \(f_i^{(r)}\) uses only the instances in its region \(\mathcal{R}_i^{(r)}\), so that it captures only the behaviour of \(x_i\) there (Algorithm~\ref{alg:gb-calm}).

\begin{algorithm}[tb]
\caption{Gradient Boosting for CALM}
\label{alg:gb-calm}
\begin{algorithmic}[1]
   \STATE Initialize $\beta_0 \leftarrow \frac{1}{N}\sum_{j=1}^N y^{(j)}$; $f_i^{(r)} \leftarrow 0$ for all $i = 1, \dots, d$ and $r = 1, \dots, R_i$
   \FOR{$m = 1$ {\bfseries to} $M$}
      \FOR{$i = 1$ {\bfseries to} $d$}
         \FOR{$r = 1$ {\bfseries to} $R_i$}
            \STATE $\mathcal{E} \leftarrow \{(x^{(j)}_i, y^{(j)} - f_{\text{CALM}}(\mathbf{x}^{(j)})) : \mathbf{x}^{(j)}_{-i} \in \mathcal{R}_i^{(r)}\}$
            \COMMENT{Filter by region $\mathcal{R}_i^{(r)}$}
            \STATE Learn shape function $S : x_i \rightarrow \mathbb{R}$ using $\mathcal{E}$
            \STATE $f_i^{(r)} \leftarrow f_i^{(r)} + \eta S$ \COMMENT{Update regional shape function}
         \ENDFOR
      \ENDFOR
      \STATE \textbf{Center:} for all $i$, let $c_i \leftarrow \frac{1}{N}\sum_{j=1}^N f_i^{(r_i(\mathbf{x}^{(j)}_{-i}))}(x_i^{(j)})$; set $f_i^{(r)} \leftarrow f_i^{(r)} - c_i$ for all $r$ and $\beta_0 \leftarrow \beta_0 + c_i$ \COMMENT{Eq.~\eqref{eq:centering}; fitted values unchanged}
   \ENDFOR
\end{algorithmic}
\end{algorithm}

\subsubsection{Theoretical analysis of Step~2}
\label{app:step3-theory}

This section formalizes the claims made in Step~2 of the main text.
We show that, under squared loss and fixed region-selection trees, (i) the
population target of Step~2 is an $L_2(P_{\mathbf{X}})$-best approximation of the true regression
function onto the CALM function class, and (ii) an idealized exact version of the
algorithm converges to the corresponding empirical risk minimizer.

Consider the CALM score function
\begin{equation}
s(\xb)
=
\beta_0
+
\sum_{i=1}^d f_i^{(r_i(\mathbf{x}_{-i}))}(x_i),
\label{eq:calm-score}
\end{equation}
where the region-selection functions $\{r_i\}$ (or equivalently the trees
$\{T_i\}$) are treated as fixed.
Let
\[
m(\xb)=\mathbb{E}[Y\mid \mathbf{X}=\xb]
\]
denote the true regression function.
We assume regression with identity link and squared loss, and $\mathbb{E}[Y^2]<\infty$, hence $m\in L_2(P_{\mathbf{X}})$.
We restrict attention to scores $s\in L_2(P_{\mathbf{X}})$.

The representation~\eqref{eq:calm-score} is not unique: every $\xb$ activates exactly one curve per feature, so adding a constant to all curves of feature $i$ and subtracting it from $\beta_0$ leaves $s$ unchanged.
As in GAMs, we fix this constant by the centering condition
\begin{equation}
\mathbb{E}\!\left[
f_i^{(r_i(\mathbf{X}_{-i}))}(X_i)
\right]=0
\quad
\text{for all } i,
\label{eq:centering}
\end{equation}
which moves the mean contribution of each feature into $\beta_0$ and leaves the level of every region inside its curve.
The condition is imposed per feature, not per region: centering each $f_i^{(r)}$ within its own region would also remove the region levels, which a single intercept cannot absorb, and would therefore shrink the class.
Constants can still move between the curves of $x_i$ and the terms of another feature whenever the indicator of a region of $x_i$ can also be expressed by those terms (e.g., a region defined by a single split on $x_j$, which $f_j$ can absorb as a step).
Such moves change neither $s$ nor any prediction difference; the results below concern $s$ and do not require the components to be unique.

Let $\mathcal H(\{T_i\})$ denote the class of all CALM score functions of the form
\eqref{eq:calm-score}.
With the trees fixed, it is a linear space, hence convex; we assume $\mathcal H(\{T_i\})\subset L_2(P_{\mathbf{X}})$ is closed.
Condition~\eqref{eq:centering} selects a representation of each $s\in\mathcal H(\{T_i\})$ and does not restrict the class.

\subsubsection{Population target}

\begin{proposition}[Population optimality]
\label{prop:population}
There exists at least one minimizer
\[
s^\star
\in
\arg\min_{s\in\mathcal H(\{T_i\})}
\mathbb{E}[(Y-s(\mathbf{X}))^2],
\]
and it satisfies
\[
s^\star
\in
\arg\min_{s\in\mathcal H(\{T_i\})}
\mathbb{E}[(m(\mathbf{X})-s(\mathbf{X}))^2].
\]
In particular,
\[
\mathbb{E}[(Y-s^\star(\mathbf{X}))^2]
-
\mathbb{E}[(Y-m(\mathbf{X}))^2]
=
\mathbb{E}[(m(\mathbf{X})-s^\star(\mathbf{X}))^2].
\]
\end{proposition}

\begin{proof}
For any measurable function $s$,
\[
Y-s(\mathbf{X})
=
(Y-m(\mathbf{X}))+(m(\mathbf{X})-s(\mathbf{X})).
\]
Taking squares and expectations yields
\[
\mathbb{E}[(Y-s(\mathbf{X}))^2]
=
\mathbb{E}[(Y-m(\mathbf{X}))^2]
+
\mathbb{E}[(m(\mathbf{X})-s(\mathbf{X}))^2],
\]
since $\mathbb{E}[Y-m(\mathbf{X})\mid \mathbf{X}]=0$ eliminates the cross-term.
The first term does not depend on $s$, so minimizing prediction error over
$\mathcal H(\{T_i\})$ is equivalent to minimizing
$\mathbb{E}[(m(\mathbf{X})-s(\mathbf{X}))^2]$.
Existence of $s^\star$ follows since $\mathcal H(\{T_i\})\subset L_2(P_{\mathbf{X}})$ is nonempty, closed, and convex.
Moreover, $s^\star$ is unique as an element of $L_2(P_{\mathbf{X}})$, being the projection of $m$ onto a closed convex set; its components are unique only up to the moves described above.
\end{proof}

\textbf{Interpretation.}
Proposition~\ref{prop:population} shows that, with fixed regions, Step~2 computes an $L_2(P_{\mathbf{X}})$-best approximation of the true regression function $m$ within the
CALM function class.
In the population, any remaining error is therefore due to the structural limitations of
the chosen regions and univariate components, not to the optimization procedure.

\subsubsection{Empirical convergence of exact Step~2}

Given i.i.d.\ data $\{(x^{(k)},y^{(k)})\}_{k=1}^N$, define the empirical risk
\[
\widehat{\mathcal R}_N(s)
=
\frac{1}{N}\sum_{k=1}^N (y^{(k)}-s(x^{(k)}))^2.
\]

\begin{proposition}[Convergence of exact cyclic regional backfitting]
\label{prop:empirical}
Assume fixed trees and squared loss.
Consider an idealized version of Step~2 that cyclically updates each
$f_i^{(r)}$ by exactly minimizing $\widehat{\mathcal R}_N$ over that function
using only samples with $r_i(\mathbf{x}^{(k)}_{-i})=r$, followed by the empirical version of~\eqref{eq:centering}: the sample mean $\frac{1}{N}\sum_{k=1}^N f_i^{(r_i(\mathbf{x}^{(k)}_{-i}))}(x_i^{(k)})$ of feature $i$'s contribution is subtracted from every curve of $x_i$ and added to $\beta_0$, so that fitted values (and hence $\widehat{\mathcal R}_N$) are unchanged.
Then $\widehat{\mathcal R}_N$ is non-increasing along the iterates, and the
procedure converges (in empirical $L_2$) to a limit $\hat s\in\arg\min_{s\in\mathcal H(\{T_i\})}\widehat{\mathcal R}_N(s)$.
Moreover, the fitted-value vector $(\hat s(\mathbf{x}^{(1)}),\ldots,\hat s(\mathbf{x}^{(N)}))$ is unique.
\end{proposition}

\begin{proof}
With fixed trees, the CALM score is linear in the collection of shape functions
$\{f_i^{(r)}\}$ evaluated on the data, and $\widehat{\mathcal R}_N$ is a convex
quadratic function of these parameters.
On the finite sample, $\widehat{\mathcal R}_N$ depends on each $f_i^{(r)}$ only through the finite vector
$\{f_i^{(r)}(x_i^{(k)}): r_i(\mathbf{x}_{-i}^{(k)})=r\}$, so the problem can be viewed as a finite-dimensional least-squares problem.
Each regional update solves the exact least-squares problem for one block of
parameters while holding the others fixed, and therefore cannot increase
$\widehat{\mathcal R}_N$.
Although regions overlap across features, all updates jointly optimize a single
global objective.
Standard results for backfitting with projection smoothers~\cite{buja1989linear, ansley1994convergence} imply
convergence to a global minimizer, and the minimizing fitted values are unique by strict convexity of the quadratic loss in the fitted-value vector.
The centering step changes no fitted value; it only fixes the constant of each feature.
\end{proof}

\textbf{Interpretation.}
This result shows that Step~2 is not a collection of independent local fits.
Instead, it performs a coupled optimization of a single additive predictor with
region-gated components, analogous to classical backfitting for GAMs~\cite{buja1989linear}.
Under exact updates, this procedure is guaranteed to converge to the best CALM
fit for the given data and fixed regions; with base learners linear in their parameters, boosting provably converges to it as well~\cite{schulte2025additive}.

\subsection{Further related work}\label{app:further-related-work}
Several models give a feature more than one univariate curve.
Neural Additive Experts~\cite{xiong2026neural} mix several expert curves per feature through a soft gate learned end to end over all features, so which curve applies cannot be read as a rule.
Locally interpretable tree boosting~\cite{hjort2024locally} fits one GAM per level of a given grouping variable, and cluster-based GAMs~\cite{huang2025cluster} fit one GAM per latent regime found from a fitted random-feature model; in both, all features share the same regions.
Threshold GAMs~\cite{ciannelli2004nonadditive} switch the curve of a chosen feature at a threshold on other variables, and factor-by-curve terms~\cite{coull2001simple} give one curve per level of a given factor, whereas CALM learns which features condition each curve.
Closest in structure are tree-based variants of varying-coefficient models~\cite{hastie1993varying}: TSVC~\cite{berger2019tree} and vcrpart~\cite{burgin2017coefficient} grow one tree per feature over other variables, but their leaves hold a slope rather than a curve, and the trees are grown by likelihood-based splits on the data rather than on the local effects of a reference model.
Model-based recursive partitioning~\cite{zeileis2008model, seibold2019generalised} and LIME-SUP~\cite{hu2018locally} instead grow a single tree shared by all varying parameters.
Higher-order additive models go beyond pairs: SIAN~\cite{enouen2022sparse} adds the feature subsets found by interaction detection as terms, and SPAM~\cite{dubey2022scalable} models all interactions up to a chosen degree with low-rank polynomials; both represent an interaction with a term over several features, whereas CALM represents it with univariate curves.
Like CALM, maidrr~\cite{henckaerts2022when}, SAFE~\cite{gosiewska2021simpler} and TabDistill~\cite{jia2026selecting} use a black box only to decide the structure of an interpretable model that is then fitted to the data, but none of them has regions.
Prior user studies found that clinicians and data scientists can misread GAM and GA$^2$M plots~\cite{hegselmann2020evaluation, kaur2020interpreting}, and that reading effort grows with the number of curves and kinks~\cite{abdul2020cogam, kruschel2025quantifying}.
Finally, equally accurate models can differ in their effects, interactions and importances~\cite{biecek2024performance, li2024exploring, cavus2025beyond, ewald2026cashomon}, which is why we report the stability of CALM's trees (Section~\ref{sec:stability}).

\newpage
\section{Formal Proofs for Interpretability Properties}
\label{app:interpretability}

This appendix provides formal characterizations and proofs for the interpretability properties discussed in Section~\ref{subsec:interpretability-properties}. For each property, we define it precisely and analyze whether it can be answered exactly under  GAM, \gamm, and \method{}.

\textbf{Note on link functions.} Properties P1--P3 are stated and proved for the additive linear predictor $\eta(\xb) = \beta_0 + \sum_i f_i^{(r_i(\xbi))}(x_i)$. With a non-identity link $g$ (Eq.~\eqref{eq:CALM}), they hold on $\eta$, the scale on which the model is additive; monotonicity (P3) also carries over to $\hat{y} = g^{-1}(\eta)$ for any increasing link $g$ (e.g., logit, log).

\begin{proposition}[Local Feature Contribution]

Let \( \hat{y} : \mathbb{R}^d \to \mathbb{R} \) be a prediction function, and fix an input \( \xb = (x_1, \dots, x_d) \in \mathbb{R}^d \). For each index \( i \in \{1, \dots, d\} \), we define the local contribution \( \phi_i(\xb) \in \mathbb{R} \) as a function intended to represent the contribution of feature \( x_i \) to the output \( \hat{y}(\xb) \). We analyze whether such a decomposition
\[
\hat{y}(\xb) = \sum_{i=1}^d \phi_i(\xb)
\]
is uniquely determined by the structure of the model.

\textbf{(A) In GAM.}
Assume the model has the additive form
\[
\hat{y}(\xb) = \sum_{j=1}^d f_j(x_j),
\]
with each \( f_j : \mathbb{R} \to \mathbb{R} \). Then for all \( i \in \{1, \dots, d\} \),
\[
\phi_i(\xb) := f_i(x_i)
\]
is well-defined, and
\[
\hat{y}(\xb) = \sum_{i=1}^d \phi_i(\xb).
\]

\textbf{(B) In \gamm.}
Assume the model includes pairwise interactions:
\[
\hat{y}(\xb) = \sum_{j=1}^d f_j(x_j) + \sum_{j<k} f_{jk}(x_j, x_k),
\]
where each \( f_j : \mathbb{R} \to \mathbb{R} \) and \( f_{jk} : \mathbb{R}^2 \to \mathbb{R} \). Then, unless every \( f_{jk} \) is additively separable, there is no decomposition
\[
\hat{y}(\xb) = \sum_{i=1}^d \phi_i(\xb)
\]
with each \( \phi_i(\xb) \) depending only on \( x_i \); decompositions whose \( \phi_i \) also depend on other features exist (e.g., assign each \( f_{jk} \) to one of its two features) but are not unique~\cite{hiabu2023unifying, fumagalli2025unifying, bordt2023shapley}.

\textbf{(C) In \method{}.}
Assume the model has the form
\[
\hat{y}(\xb) = \sum_{j=1}^d f_j^{(r_j(\xbj))}(x_j),
\]
where for each \( j \), \( r_j : \mathbb{R}^{d-1} \to \{1, \dots, R_j\} \) is a region selector and \( f_j^{(r)} : \mathbb{R} \to \mathbb{R} \) is a univariate shape function for region \( r \). Here \( \xbj := \xb \setminus x_j \in \mathbb{R}^{d-1} \).

Then for all \( i \in \{1, \dots, d\} \),
\[
\phi_i(\xb) := f_i^{(r_i(\xbi))}(x_i)
\]
is well-defined, and
\[
\hat{y}(\xb) = \sum_{i=1}^d \phi_i(\xb).
\]
\end{proposition}

\begin{proof}
We examine each case separately.

\textbf{(A) In GAM.} By direct substitution, the prediction is an additive sum of univariate functions, each depending only on a single feature \( x_j \). Thus the contribution of feature \( x_i \) is read off the model as \( f_i(x_i) \), and the decomposition is exact. 

\textbf{(B) In \gamm.} Suppose, for contradiction, that there exists a decomposition
\[
\hat{y}(\xb) = \sum_{i=1}^d \phi_i(\xb),
\]
where each \( \phi_i : \mathbb{R}^d \to \mathbb{R} \) represents the contribution of feature \( x_i \) and depends only on \( x_i \), i.e., \( \phi_i(\xb) = \phi_i(x_i) \).

Then each interaction term \( f_{jk}(x_j, x_k) \) must be split between \( \phi_j \) and \( \phi_k \) in such a way that the total sum remains correct and additive over individual features. This is only possible if \( f_{jk} \) is additively separable, i.e., if there exist functions \( g_j, g_k : \mathbb{R} \to \mathbb{R} \) such that
\[
f_{jk}(x_j, x_k) = g_j(x_j) + g_k(x_k).
\]
However, the model does not constrain \( f_{jk} \) to be separable. In the general case, \( f_{jk} \) is non-separable and depends jointly on \( x_j \) and \( x_k \).

Therefore, no decomposition \( \sum_i \phi_i(x_i) \) can reproduce the prediction \( \hat{y}(\xb) \) using only univariate terms. The contributions \( \phi_i(\xb) \) are not uniquely determined by the model, and must rely on external assumptions or post-hoc attribution methods~\cite{gosiewska2019do}.

 \textbf{(C) In \method{}.} For a fixed input \( \xb \), each region function \( r_j(\xbj) \) returns a unique region index in \( \{1, \dots, R_j\} \). The corresponding function \( f_j^{(r_j(\xbj))} \) is applied to \( x_j \), yielding a scalar. Thus, the contribution of feature \( x_i \) is precisely defined as:
\[
\phi_i(\xb) := f_i^{(r_i(\xbi))}(x_i),
\]
and the model prediction is recovered by summing over all features:
\[
\hat{y}(\xb) = \sum_{i=1}^d f_i^{(r_i(\xbi))}(x_i) = \sum_{i=1}^d \phi_i(\xb). 
\]
\end{proof}
As in any additive model, the components, and hence the \( \phi_i \), are defined only up to moves of constants that leave \( \hat{y} \) unchanged, which the centering of Eq.~\eqref{eq:centering} fixes up to the moves described in Appendix~\ref{app:step3-theory}; the point of P1 is that in \gam{} and \method{} the contributions are read off the fitted model, whereas in \gamm{} they must be constructed post hoc.

The second property asks: \emph{How does the prediction change if we perturb a single feature \( x_i \) while keeping all others fixed?} The answer differs across GAM, \gamm, and \method{}, as shown below.

\begin{proposition}[Regional Feature Sensitivity]
Let \( \hat{y} : \mathbb{R}^d \to \mathbb{R} \) be a prediction function. Fix an input \( \xb \in \mathbb{R}^d \), an index \( i \in \{1, \dots, d\} \), and a perturbation \( \varepsilon > 0 \). Define the prediction change under perturbation of \( x_i \) as
\[
\Delta \hat{y} := \hat{y}(\xb + \varepsilon \mathbf{e}_i) - \hat{y}(\xb).
\]
We analyze how \( \Delta \hat{y} \) can be computed in GAM, \gamm, and \method{}.

\textbf{(A) In GAM.}
Assume the model has the additive form
\[
\hat{y}(\xb) = \sum_{j=1}^d f_j(x_j),
\]
with each \( f_j : \mathbb{R} \to \mathbb{R} \) univariate. Then
\[
\Delta \hat{y} = f_i(x_i + \varepsilon) - f_i(x_i).
\]

\textbf{(B) In \gamm.}
We show that in \gamm, the change \( \Delta \hat{y} \) caused by perturbing \( x_i \) depends on the values of the features \( x_j \) that interact with \( x_i \), and thus, in general, cannot be computed from \( x_i \) alone.

\textbf{(C) In \method{}:} We show that the change in prediction \( \Delta \hat{y} := \hat{y}(\xb + \varepsilon \mathbf{e}_i) - \hat{y}(\xb) \) is computable from \( x_i \) alone if no region transitions are triggered in other features. Otherwise, the change depends in general on \( \xbi \).

\end{proposition}

\begin{proof}

\textit{Note:} The proposition is stated for a fixed input $\xb$ (single-valued); the set-valued definition of P2 in the main text arises by ranging over all possible values of $\xbi$, each of which selects a region $r_i(\xbi) \in \{1, \dots, R_i\}$ and thus one element of the set $\{\Delta f_i^{(r)}\}_{r=1}^{R_i}$.

\textbf{(A) In GAM.}
Since all other features remain unchanged, and their contributions are independent of \( x_i \), we have
\[
\hat{y}(\xb + \varepsilon \mathbf{e}_i) = \sum_{j \neq i} f_j(x_j) + f_i(x_i + \varepsilon),
\]
\[
\hat{y}(\xb) = \sum_{j \neq i} f_j(x_j) + f_i(x_i),
\]
and therefore
\[
\Delta \hat{y} = f_i(x_i + \varepsilon) - f_i(x_i). 
\]

\textbf{(B) In \gamm.}
Let \( \mathcal{J} := \{ j \ne i \mid f_{ij}(x_i, x_j) \text{ is present in the model} \} \), i.e., the set of features that interact with \( x_i \). Since only \( x_i \) is perturbed, all other features remain unchanged. The only affected terms are the univariate function \( f_i(x_i) \) and the interaction terms \( f_{ij}(x_i, x_j) \) for each \( j \in \mathcal{J} \). All other terms in the model remain constant.

The prediction after perturbation is:
\[
\hat{y}(\xb + \varepsilon \mathbf{e}_i) = f_i(x_i + \varepsilon) + \sum_{j \ne i} f_j(x_j) + \sum_{j \in \mathcal{J}} f_{ij}(x_i + \varepsilon, x_j) + C,
\]
and before perturbation:
\[
\hat{y}(\xb) = f_i(x_i) + \sum_{j \ne i} f_j(x_j) + \sum_{j \in \mathcal{J}} f_{ij}(x_i, x_j) + C,
\]
where \( C := \sum_{j<k,\, j,k \ne i} f_{jk}(x_j, x_k) \) denotes the interaction terms not involving \( x_i \), which are unaffected.

Subtracting, we obtain:
\[
\Delta \hat{y} = f_i(x_i + \varepsilon) - f_i(x_i) + \sum_{j \in \mathcal{J}} \left[ f_{ij}(x_i + \varepsilon, x_j) - f_{ij}(x_i, x_j) \right].
\]

Since this expression depends on \( x_j \) for \( j \in \mathcal{J} \), it cannot, in general, be computed from \( x_i \) alone. 

\textbf{(C) In \method{}.}
Assume the model has the form
\[
\hat{y}(\xb) = \sum_{j=1}^d f_j^{(r_j(\xbj))}(x_j),
\]
where \( r_j : \mathbb{R}^{d-1} \to \{1, \dots, R_j\} \) assigns a region index to each feature \( j \) based on the context \( \xbj := \xb \setminus x_j \), and \( f_j^{(r)} : \mathbb{R} \to \mathbb{R} \) is the shape function for region \( r \).

Let \( \xbi = \xb \setminus x_i \), and define the perturbation \( x_i \mapsto x_i + \varepsilon \). Then
\[
\Delta \hat{y} = f_i^{(r_i(\xbi))}(x_i + \varepsilon) - f_i^{(r_i(\xbi))}(x_i)
+ \sum_{j \ne i} \left[ f_j^{(r_j(\xbj^{+\varepsilon}))}(x_j) - f_j^{(r_j(\xbj))}(x_j) \right],
\]
where \( \xbj^{+\varepsilon} := \xbj \) with \( x_i \mapsto x_i + \varepsilon \) (since \( x_i \in \xbj \) for all \( j \ne i \)).

\textit{Case 1: No region transitions in other features.}
Assume that for all \( j \ne i \), we have \( r_j(\xbj^{+\varepsilon}) = r_j(\xbj) \). Then:
\[
\Delta \hat{y} = f_i^{(r_i(\xbi))}(x_i + \varepsilon) - f_i^{(r_i(\xbi))}(x_i).
\]
This case is fully interpretable from the function \( f_i^{(r)} \) alone.
It holds for every \( \xbi \) when no tree \( T_j \), \( j \ne i \), splits on \( x_i \) at a threshold in \( [x_i, x_i + \varepsilon) \), i.e., when the perturbation crosses no vertical line of the plot.

\textit{Case 2: Region transitions occur in other features.}
If for some \( j \ne i \), the region function changes: \( r_j(\xbj^{+\varepsilon}) \ne r_j(\xbj) \), then the additional terms contribute:
\[
\Delta f_j := f_j^{(r_j(\xbj^{+\varepsilon}))}(x_j) - f_j^{(r_j(\xbj))}(x_j).
\]
Hence, the total change includes not only the direct shift in \( f_i \), but also discrete region-based changes in other features and the value of \( \Delta \hat{y} \) cannot, in general, be recovered from \( x_i \) alone. 
\end{proof}

The third property asks whether increasing a feature always increases the model's prediction, regardless of the values of other features. In other words, does the model treat the feature as globally monotonic?

\begin{proposition}[Global Feature Monotonicity]

Let \( \hat{y} : \mathbb{R}^d \to \mathbb{R} \) be a prediction function, and fix a feature index \( i \in \{1, \dots, d\} \). Define the global monotonicity condition for feature \( x_i \) as follows: the model is said to be \emph{globally increasing in \( x_i \)} if for all \( \xb \in \mathbb{R}^d \) and all \( \delta > 0 \),
\[
\hat{y}(\xb + \delta \mathbf{e}_i) \ge \hat{y}(\xb),
\]
and \emph{strictly increasing} if the inequality is strict.

We analyze whether this property can be verified from the structure of the model.

\textbf{(A) In GAM.}
Assume the model has the form:
\[
\hat{y}(\xb) = \sum_{j=1}^d f_j(x_j),
\]
with each \( f_j : \mathbb{R} \to \mathbb{R} \).

Then the model is globally increasing in \( x_i \) if and only if \( f_i \) is monotonically increasing.

\textbf{(B) In \gamm.}
Assume the model has the form:
\[
\hat{y}(\xb) = \sum_{j=1}^d f_j(x_j) + \sum_{j<k} f_{jk}(x_j, x_k).
\]

Then global monotonicity in \( x_i \) cannot be determined from \( f_i \) alone.

\textbf{(C) In \method{}.}
Assume the model has the form:
\[
\hat{y}(\xb) = \sum_{j=1}^d f_j^{(r_j(\xbj))}(x_j),
\]
where \( r_j : \mathbb{R}^{d-1} \to \{1, \dots, R_j\} \) is a region selector, and \( \xbj := \xb \setminus x_j \).

Then global monotonicity in \( x_i \) holds  if the following two conditions are satisfied:
\begin{enumerate}
  \item For all regions \( r \in \{1, \dots, R_i\} \), the function \( f_i^{(r)} : \mathbb{R} \to \mathbb{R} \) is increasing.
  \item For all \( j \ne i \) and all values of the features other than \( x_i \), the function \( x_i \mapsto f_j^{(r_j(\xbj))}(x_j) \) is non-decreasing; i.e., region transitions caused by changing \( x_i \) do not decrease \( f_j \)'s contribution.
\end{enumerate}

\end{proposition}

\begin{proof}

\textbf{(A) In GAM.}
Since the model is additive and each term depends only on a single variable, we have:
\[
\hat{y}(\xb + \delta \mathbf{e}_i) - \hat{y}(\xb) = f_i(x_i + \delta) - f_i(x_i).
\]
Thus, \( \hat{y}(\xb + \delta \mathbf{e}_i) \ge \hat{y}(\xb) \) if and only if \( f_i \) is increasing. 

\textbf{(B) In \gamm.}  As in case (A), \( f_i(x_i + \delta) - f_i(x_i) \) gives the direct contribution. However, for each \( j \ne i \), the term \( f_{ij}(x_i, x_j) \) also contributes. The total change is:
\[
\hat{y}(\xb + \delta \mathbf{e}_i) - \hat{y}(\xb) = f_i(x_i + \delta) - f_i(x_i)
+ \sum_{j \ne i} \left[ f_{ij}(x_i + \delta, x_j) - f_{ij}(x_i, x_j) \right].
\]
The sign of this expression depends on the values of \( x_j \), and therefore cannot be determined from \( f_i \) alone. Consequently, verifying global monotonicity in \( x_i \) requires knowing the full interaction structure and input values. 

\textbf{(C) In \method{}.}
Let \( \xb \in \mathbb{R}^d \) and \( \delta > 0 \). Define:
\[
\Delta \hat{y} := \hat{y}(\xb + \delta \mathbf{e}_i) - \hat{y}(\xb).
\]
Then:
\[
\Delta \hat{y} =
f_i^{(r_i(\xbi))}(x_i + \delta) - f_i^{(r_i(\xbi))}(x_i)
+ \sum_{j \ne i} \left[
f_j^{(r_j(\xbj^{+\delta}))}(x_j) - f_j^{(r_j(\xbj))}(x_j)
\right],
\]
where \( \xbj^{+\delta} \) is obtained by replacing \( x_i \mapsto x_i + \delta \) in \( \xbj \).

The first term is non-negative if \( f_i^{(r)} \) is increasing for all regions \( r \), and the region \( r_i(\xbi) \) is fixed.

The second term is a sum over changes in other features' contributions due to possible changes in their regions \( r_j \); under condition 2, each of them is non-negative.

Therefore, if both conditions are satisfied, every term is non-negative, implying \( \Delta \hat{y} \ge 0 \).
The conditions are sufficient but not necessary: a downward jump can be offset by an increase of \( f_i^{(r)} \) at the same point. 

\end{proof}

\newpage
\section{Detailed Experiments}\label{app:detailed-experiments}
\renewcommand{\floatpagefraction}{0.85}\renewcommand{\topfraction}{0.9}\renewcommand{\bottomfraction}{0.6}\renewcommand{\textfraction}{0.07}

\subsection{Experimental Setup Details}
\textbf{Reporting convention.}
As in Section~\ref{sec:empirical-evaluation}, we report mean $\pm$ standard deviation over 5-fold cross-validation.

\textbf{Implementations.} We used open-source implementations for all baselines: \href{https://xgboost.readthedocs.io/en/release_3.2.0/python/index.html}{XGBoost}~\cite{chen2016xgboost}, \href{https://www.tensorflow.org/}{tensorflow}~\cite{abadi2015tensorflow} (NAM~\cite{agarwal2021neural}), \href{https://interpret.ml/docs/}{interpretml}~\cite{nori2019interpretml} (EBM, EB$^2$M), and the official repositories for \href{https://github.com/zzzace2000/nodegam}{NODE-GA$^2$M}~\cite{chang2022node} and \href{https://github.com/ZebinYang/gaminet}{GAMI‐Net}~\cite{yang2020gami}.

\label{app:datasets-model-details}
\begin{table}[!htbp]
\centering
\small
\caption{The 24 real-world datasets: number of rows, number of features and source.}
\label{tab:datasets}
\resizebox{\textwidth}{!}{
\begin{tabular}{lrrl c lrrl}
\toprule
\multicolumn{4}{c}{\textbf{Classification (10)}} & & \multicolumn{4}{c}{\textbf{Regression (14)}} \\
\cmidrule(lr){1-4}\cmidrule(lr){6-9}
Dataset & Rows & Features & Source & & Dataset & Rows & Features & Source \\
\midrule
Adult & 45\,222 & 13 & UCI~\cite{uci_ml, kohavi1996scaling} & & Ailerons & 13\,750 & 39 & OpenML~\cite{OpenML2013} \\
Appendicitis & 106 & 7 & PMLB~\cite{romano2021pmlb} & & Airfoil & 1\,503 & 5 & OpenML~\cite{OpenML2013, brooks1989airfoil} \\
Bank & 45\,211 & 16 & UCI~\cite{uci_ml, moro2014data} & & Bike Sharing & 17\,379 & 11 & UCI~\cite{uci_ml, fanaee2014event} \\
Churn & 5\,000 & 19 & PMLB~\cite{romano2021pmlb} & & California Housing & 20\,640 & 8 & \cite{pace1997sparse} \\
COMPAS & 6\,167 & 9 & \cite{propublica2016compas} & & CCPP & 9\,568 & 4 & UCI~\cite{uci_ml, tufekci2014prediction} \\
HELOC & 10\,459 & 23 & \cite{fico2018heloc} & & Elevators & 16\,599 & 18 & OpenML~\cite{OpenML2013} \\
Magic & 19\,020 & 10 & UCI~\cite{uci_ml, bock2004methods} & & Energy & 19\,735 & 28 & \cite{luism78_energydata, candanedo2017energy} \\
MIMIC2 & 24\,508 & 17 & \cite{saeed2011mimic} & & No2 & 500 & 7 & OpenML~\cite{OpenML2013} \\
Phoneme & 5\,404 & 5 & PMLB~\cite{romano2021pmlb} & & Parkinsons Motor & 5\,875 & 19 & UCI~\cite{uci_ml, tsanas2010accurate} \\
SPECTF & 349 & 44 & PMLB~\cite{romano2021pmlb} & & Parkinsons Total & 5\,875 & 19 & UCI~\cite{uci_ml, tsanas2010accurate} \\
 & & &  & & Sensory & 576 & 11 & OpenML~\cite{OpenML2013} \\
 & & &  & & Seoul Bike & 8\,465 & 14 & UCI~\cite{uci_ml, sathishkumar2020using} \\
 & & &  & & Skill Craft & 3\,338 & 19 & UCI~\cite{uci_ml, thompson2013video} \\
 & & &  & & Wine & 6\,497 & 11 & UCI~\cite{uci_ml, cortez2009modeling} \\
\bottomrule
\end{tabular}}
\end{table}

\textbf{Dataset details.}
Table~\ref{tab:datasets} lists the 10 classification and 14 regression datasets with their size and source.

\textbf{Data preprocessing.}
All input features are standardized using z-score normalization. For regression tasks, the target variable $y$ is also standard scaled. The only exception is the GAMI-Net model, which requires input features to be scaled using min-max normalization.
For RMSE, predictions and targets are inverse-transformed to the original scale prior to evaluation.

\textbf{Model configurations.}

The black box is XGBoost with 300 boosting rounds, a learning rate of 0.1 and the log-loss evaluation metric for classification; \method{} uses the same model as its reference model.

The GAMs are NAM and EBM without interactions. The Neural Additive Model (NAM) builds one independent multilayer perceptron per feature; each sub-network has three hidden layers of 100,100 and 10 ReLU units followed by a linear output, and the additive sum is passed through a sigmoid for binary classification. Training uses Adam (learning rate 0.001), for 10 epochs and a mini-batch size of 32. The Explainable Boosting Machine (EBM) is used without interaction terms by explicitly setting $n\_interactions=0$.

To capture pairwise interactions we experiment with three GA\textsuperscript{2}M variants. EB\textsuperscript{2}M extends the EBM with $\lceil 0.9\,d\rceil$ pairwise interactions, i.e., \texttt{interactions=0.9} in interpret; NODE‑GA\textsuperscript{2}M activates its interaction mode via the ga2m=1 flag and restricts training to a five‑minute time limit to ensure parity with the other models; this cap was chosen to match the practical training budget of competing methods, and we observed that NODE‑GA\textsuperscript{2}M's training loss had largely plateaued within this window; GAMI‑Net is run with the default hyper‑parameters.

For \method{}, we use trees of maximum depth $d_{\max} = 2$, a heterogeneity drop threshold of $\epsilon = 0.05$ on the standard-deviation scale of Eq.~\eqref{eq:heter-definition}, a selection threshold of $\gamma = 0.01$, and no cap $K$ on the number of conditional interactions (Section~\ref{subsec:selection}). Heterogeneity is measured with PDP, as implemented in Effector~\cite{gkolemis2024effector}; the detector evaluates 20 candidate split points per numerical conditioning feature and one-vs-rest splits for categorical ones, with at least 10 instances per leaf. The shape functions of Step~2 are fitted with an EBM without interactions on a masked design: one column per feature and region, missing outside its region. Each curve is its column's shape function plus the missing-value scores of the feature's other columns, so the gaps between curves survive and the centering stays per feature.

\textbf{Computer resources.} XGBoost, EBM, EB$^2$M and \method{} were run on a laptop (Intel Core i7-10750H, 10 CPU threads, 16 GB of RAM); NAM, NODE-GA$^2$M and GAMI-Net are reported from an earlier run on a server (Intel Core i9-10900X, 128 GB of RAM); EB$^2$M's runtimes in Tables~\ref{tab:clf-time-def} and~\ref{tab:regr-time-def} are from the earlier run too. No GPU was used.

\subsection{Synthetic Examples}\label{app:synthetic}
Table~\ref{tab:explanatory-example-r2} gives the accuracy that Figure~\ref{fig:shape_functions} prints under each case, with its spread over the folds.

\begin{table}[!htbp]
\centering
\caption{Synthetic examples: $R^2$ comparison
(mean $\pm$ std. dev. over 5-fold CV).}
\label{tab:explanatory-example-r2}
\begin{tabular}{lccc}
\toprule
\textbf{Dataset} & GAM & \gamm & \method{} \\
\midrule
Case 1 & $0.731_{\pm0.027}$ & $0.979_{\pm0.010}$ & $0.991_{\pm0.006}$ \\
Case 2 & $0.476_{\pm0.061}$ & $0.711_{\pm0.039}$ & $0.990_{\pm0.008}$ \\
Case 3 & $0.531_{\pm0.023}$ & $0.972_{\pm0.006}$ & $0.958_{\pm0.012}$ \\
\bottomrule
\end{tabular}
\end{table}

\subsection{Main Results, per Dataset}\label{app:main-tables}
Tables~\ref{tab:clf-acc} and~\ref{tab:regr-rmse} give the comparison of Figure~\ref{fig:dumbbell} for every baseline. In regression the \gamm's extra gain over \method{} is beyond the band on eight datasets and exceeds five points on seven of them (Airfoil, Bike Sharing, Seoul Bike, California Housing, Elevators, Energy, CCPP); on Parkinsons Motor \method{} is ahead, and on the remaining five the two are within a fold's spread. In classification every difference but Phoneme, Magic and Bank is inside the band.

\begin{table}[!htbp]
\caption{Classification accuracy vs.\ baselines (mean $\pm$ std.\ dev.\ over 5-fold CV).
\textbf{\ding{51}}/\textbf{\ding{55}}: higher/lower than \method{} by more than one std.\ dev.\ (the larger of the two);
\textbf{$\mathbf{=}$}: within one.}
\label{tab:clf-acc}
\centering
\resizebox{\textwidth}{!}{
\begin{tabular}{lcccccccc}
\toprule
\textbf{Dataset} & \textbf{BB} & \multicolumn{2}{c}{\textbf{GAM}} & \textbf{\method{}} & \multicolumn{3}{c}{\textbf{GA$^2$M}} \\
\cmidrule(lr){2-2} \cmidrule(lr){3-4} \cmidrule(lr){6-8}
& \textbf{XGB} & \textbf{NAM} & \textbf{EBM} & & \textbf{EB$^2$M} & \textbf{NODE} & \textbf{GAMI} \\
\midrule
Adult & $0.870 \text{\scriptsize$\pm$0.002}$ $\mathbf{=}$ & $0.851 \text{\scriptsize$\pm$0.002}$ \ding{55} & $0.870 \text{\scriptsize$\pm$0.002}$ $\mathbf{=}$ & $0.870 \text{\scriptsize$\pm$0.002}$ & $0.871 \text{\scriptsize$\pm$0.003}$ $\mathbf{=}$ & $0.850 \text{\scriptsize$\pm$0.002}$ \ding{55} & $0.857 \text{\scriptsize$\pm$0.003}$ \ding{55} \\
COMPAS & $0.664 \text{\scriptsize$\pm$0.009}$ \ding{55} & $0.682 \text{\scriptsize$\pm$0.016}$ $\mathbf{=}$ & $0.683 \text{\scriptsize$\pm$0.011}$ $\mathbf{=}$ & $0.678 \text{\scriptsize$\pm$0.013}$ & $0.684 \text{\scriptsize$\pm$0.012}$ $\mathbf{=}$ & $0.667 \text{\scriptsize$\pm$0.013}$ $\mathbf{=}$ & $0.686 \text{\scriptsize$\pm$0.013}$ $\mathbf{=}$ \\
HELOC & $0.717 \text{\scriptsize$\pm$0.013}$ $\mathbf{=}$ & $0.723 \text{\scriptsize$\pm$0.011}$ $\mathbf{=}$ & $0.728 \text{\scriptsize$\pm$0.013}$ $\mathbf{=}$ & $0.729 \text{\scriptsize$\pm$0.015}$ & $0.729 \text{\scriptsize$\pm$0.011}$ $\mathbf{=}$ & $0.715 \text{\scriptsize$\pm$0.008}$ $\mathbf{=}$ & $0.725 \text{\scriptsize$\pm$0.012}$ $\mathbf{=}$ \\
MIMIC2 & $0.890 \text{\scriptsize$\pm$0.003}$ \ding{51} & $0.886 \text{\scriptsize$\pm$0.001}$ $\mathbf{=}$ & $0.886 \text{\scriptsize$\pm$0.003}$ $\mathbf{=}$ & $0.886 \text{\scriptsize$\pm$0.003}$ & $0.887 \text{\scriptsize$\pm$0.002}$ $\mathbf{=}$ & $0.888 \text{\scriptsize$\pm$0.002}$ $\mathbf{=}$ & $0.884 \text{\scriptsize$\pm$0.003}$ $\mathbf{=}$ \\
Appendicitis & $0.859 \text{\scriptsize$\pm$0.073}$ $\mathbf{=}$ & $0.848 \text{\scriptsize$\pm$0.056}$ $\mathbf{=}$ & $0.868 \text{\scriptsize$\pm$0.046}$ $\mathbf{=}$ & $0.841 \text{\scriptsize$\pm$0.090}$ & $0.859 \text{\scriptsize$\pm$0.077}$ $\mathbf{=}$ & $0.849 \text{\scriptsize$\pm$0.016}$ $\mathbf{=}$ & $0.764 \text{\scriptsize$\pm$0.074}$ $\mathbf{=}$ \\
Phoneme & $0.900 \text{\scriptsize$\pm$0.004}$ \ding{51} & $0.808 \text{\scriptsize$\pm$0.002}$ \ding{55} & $0.827 \text{\scriptsize$\pm$0.007}$ \ding{55} & $0.849 \text{\scriptsize$\pm$0.005}$ & $0.870 \text{\scriptsize$\pm$0.006}$ \ding{51} & $0.869 \text{\scriptsize$\pm$0.003}$ \ding{51} & $0.876 \text{\scriptsize$\pm$0.009}$ \ding{51} \\
SPECTF & $0.857 \text{\scriptsize$\pm$0.028}$ $\mathbf{=}$ & $0.839 \text{\scriptsize$\pm$0.030}$ $\mathbf{=}$ & $0.894 \text{\scriptsize$\pm$0.037}$ $\mathbf{=}$ & $0.865 \text{\scriptsize$\pm$0.030}$ & $0.871 \text{\scriptsize$\pm$0.029}$ $\mathbf{=}$ & $0.820 \text{\scriptsize$\pm$0.067}$ $\mathbf{=}$ & $0.874 \text{\scriptsize$\pm$0.016}$ $\mathbf{=}$ \\
Magic & $0.885 \text{\scriptsize$\pm$0.004}$ \ding{51} & $0.850 \text{\scriptsize$\pm$0.006}$ \ding{55} & $0.858 \text{\scriptsize$\pm$0.004}$ \ding{55} & $0.865 \text{\scriptsize$\pm$0.002}$ & $0.875 \text{\scriptsize$\pm$0.004}$ \ding{51} & $0.876 \text{\scriptsize$\pm$0.004}$ \ding{51} & $0.874 \text{\scriptsize$\pm$0.003}$ \ding{51} \\
Bank & $0.910 \text{\scriptsize$\pm$0.004}$ \ding{51} & $0.901 \text{\scriptsize$\pm$0.003}$ \ding{55} & $0.903 \text{\scriptsize$\pm$0.002}$ $\mathbf{=}$ & $0.905 \text{\scriptsize$\pm$0.002}$ & $0.909 \text{\scriptsize$\pm$0.001}$ \ding{51} & $0.903 \text{\scriptsize$\pm$0.001}$ $\mathbf{=}$ & $0.908 \text{\scriptsize$\pm$0.003}$ $\mathbf{=}$ \\
Churn & $0.959 \text{\scriptsize$\pm$0.005}$ $\mathbf{=}$ & $0.885 \text{\scriptsize$\pm$0.009}$ \ding{55} & $0.887 \text{\scriptsize$\pm$0.005}$ \ding{55} & $0.958 \text{\scriptsize$\pm$0.005}$ & $0.959 \text{\scriptsize$\pm$0.009}$ $\mathbf{=}$ & $0.954 \text{\scriptsize$\pm$0.009}$ $\mathbf{=}$ & $0.952 \text{\scriptsize$\pm$0.007}$ $\mathbf{=}$ \\
\midrule
\textbf{W/D/L} & \textbf{4/5/1} & \textbf{0/5/5} & \textbf{0/7/3} &  & \textbf{3/7/0} & \textbf{2/7/1} & \textbf{2/7/1} \\
\bottomrule

\end{tabular}
}
\end{table}

\begin{table}[!htbp]
\caption{Regression RMSE vs.\ baselines (mean $\pm$ std.\ dev.\ over 5-fold CV).
\textbf{\ding{51}}/\textbf{\ding{55}}: lower/higher than \method{} by more than one std.\ dev.\ (the larger of the two);
\textbf{$\mathbf{=}$}: within one.
Elevators and Ailerons are scaled; NAM, NODE-GA$^2$M and GAMI-Net there are from an earlier run that kept one significant digit, too few to rank them against \method{}, so they count as draws.}
\label{tab:regr-rmse}
\centering
\resizebox{\textwidth}{!}{
\begin{tabular}{lcccccccc}
\toprule
\textbf{Dataset} & \textbf{BB} & \multicolumn{2}{c}{\textbf{GAM}} & \textbf{\method{}} & \multicolumn{3}{c}{\textbf{GA$^2$M}} \\
\cmidrule(lr){2-2} \cmidrule(lr){3-4} \cmidrule(lr){6-8}
& \textbf{XGB} & \textbf{NAM} & \textbf{EBM} & & \textbf{EB$^2$M} & \textbf{NODE} & \textbf{GAMI} \\
\midrule
Bike Sharing & $39.35 \text{\scriptsize$\pm$1.38}$ \ding{51} & $101.97 \text{\scriptsize$\pm$1.31}$ \ding{55} & $100.17 \text{\scriptsize$\pm$0.99}$ \ding{55} & $62.40 \text{\scriptsize$\pm$1.99}$ & $54.83 \text{\scriptsize$\pm$1.14}$ \ding{51} & $54.47 \text{\scriptsize$\pm$1.58}$ \ding{51} & $53.44 \text{\scriptsize$\pm$1.90}$ \ding{51} \\
California Housing & $0.45 \text{\scriptsize$\pm$0.01}$ \ding{51} & $0.61 \text{\scriptsize$\pm$0.01}$ \ding{55} & $0.55 \text{\scriptsize$\pm$0.01}$ \ding{55} & $0.52 \text{\scriptsize$\pm$0.01}$ & $0.47 \text{\scriptsize$\pm$0.01}$ \ding{51} & $0.50 \text{\scriptsize$\pm$0.01}$ \ding{51} & $0.51 \text{\scriptsize$\pm$0.04}$ $\mathbf{=}$ \\
Parkinsons Motor & $1.44 \text{\scriptsize$\pm$0.09}$ \ding{51} & $6.11 \text{\scriptsize$\pm$0.16}$ \ding{55} & $4.11 \text{\scriptsize$\pm$0.04}$ \ding{55} & $2.12 \text{\scriptsize$\pm$0.08}$ & $2.57 \text{\scriptsize$\pm$0.06}$ \ding{55} & $3.49 \text{\scriptsize$\pm$0.31}$ \ding{55} & $2.76 \text{\scriptsize$\pm$0.31}$ \ding{55} \\
Parkinsons Total & $1.86 \text{\scriptsize$\pm$0.08}$ \ding{51} & $7.90 \text{\scriptsize$\pm$0.10}$ \ding{55} & $4.73 \text{\scriptsize$\pm$0.06}$ \ding{55} & $3.00 \text{\scriptsize$\pm$0.12}$ & $3.02 \text{\scriptsize$\pm$0.07}$ $\mathbf{=}$ & $4.60 \text{\scriptsize$\pm$0.53}$ \ding{55} & $3.81 \text{\scriptsize$\pm$0.66}$ \ding{55} \\
Seoul Bike & $209.6 \text{\scriptsize$\pm$3.47}$ \ding{51} & $320.2 \text{\scriptsize$\pm$4.38}$ \ding{55} & $300.0 \text{\scriptsize$\pm$3.68}$ \ding{55} & $251.4 \text{\scriptsize$\pm$4.65}$ & $234.9 \text{\scriptsize$\pm$2.58}$ \ding{51} & $231.0 \text{\scriptsize$\pm$4.23}$ \ding{51} & $245.82 \text{\scriptsize$\pm$8.45}$ $\mathbf{=}$ \\
Wine & $0.62 \text{\scriptsize$\pm$0.01}$ \ding{51} & $0.72 \text{\scriptsize$\pm$0.01}$ \ding{55} & $0.70 \text{\scriptsize$\pm$0.01}$ $\mathbf{=}$ & $0.70 \text{\scriptsize$\pm$0.01}$ & $0.68 \text{\scriptsize$\pm$0.01}$ \ding{51} & $0.69 \text{\scriptsize$\pm$0.01}$ $\mathbf{=}$ & $0.71 \text{\scriptsize$\pm$0.00}$ $\mathbf{=}$ \\
Energy & $67.97 \text{\scriptsize$\pm$2.58}$ \ding{51} & $89.80 \text{\scriptsize$\pm$2.67}$ \ding{55} & $83.75 \text{\scriptsize$\pm$2.61}$ $\mathbf{=}$ & $82.85 \text{\scriptsize$\pm$2.26}$ & $76.11 \text{\scriptsize$\pm$2.72}$ \ding{51} & $82.18 \text{\scriptsize$\pm$2.53}$ $\mathbf{=}$ & $180.1 \text{\scriptsize$\pm$84.09}$ \ding{55} \\
CCPP & $3.09 \text{\scriptsize$\pm$0.09}$ \ding{51} & $4.21 \text{\scriptsize$\pm$0.05}$ \ding{55} & $3.65 \text{\scriptsize$\pm$0.07}$ $\mathbf{=}$ & $3.65 \text{\scriptsize$\pm$0.07}$ & $3.39 \text{\scriptsize$\pm$0.08}$ \ding{51} & $3.97 \text{\scriptsize$\pm$0.07}$ \ding{55} & $3.92 \text{\scriptsize$\pm$0.07}$ \ding{55} \\
Elevators ($\times 10^{-3}$) & $2.15 \text{\scriptsize$\pm$0.05}$ \ding{51} & $2.0 \text{\scriptsize$\pm$0}$ $\mathbf{=}$ & $2.38 \text{\scriptsize$\pm$0.08}$ $\mathbf{=}$ & $2.30 \text{\scriptsize$\pm$0.10}$ & $2.08 \text{\scriptsize$\pm$0.06}$ \ding{51} & $2.0 \text{\scriptsize$\pm$0}$ $\mathbf{=}$ & $2.0 \text{\scriptsize$\pm$0}$ $\mathbf{=}$ \\
No2 & $0.47 \text{\scriptsize$\pm$0.03}$ $\mathbf{=}$ & $0.50 \text{\scriptsize$\pm$0.02}$ $\mathbf{=}$ & $0.49 \text{\scriptsize$\pm$0.03}$ $\mathbf{=}$ & $0.49 \text{\scriptsize$\pm$0.03}$ & $0.47 \text{\scriptsize$\pm$0.03}$ $\mathbf{=}$ & $0.52 \text{\scriptsize$\pm$0.02}$ $\mathbf{=}$ & $0.50 \text{\scriptsize$\pm$0.03}$ $\mathbf{=}$ \\
Sensory & $0.51 \text{\scriptsize$\pm$0.03}$ \ding{55} & $0.48 \text{\scriptsize$\pm$0.01}$ \ding{55} & $0.48 \text{\scriptsize$\pm$0.01}$ \ding{55} & $0.45 \text{\scriptsize$\pm$0.02}$ & $0.45 \text{\scriptsize$\pm$0.02}$ $\mathbf{=}$ & $0.49 \text{\scriptsize$\pm$0.01}$ \ding{55} & $0.46 \text{\scriptsize$\pm$0.03}$ $\mathbf{=}$ \\
Airfoil & $1.54 \text{\scriptsize$\pm$0.10}$ \ding{51} & $4.80 \text{\scriptsize$\pm$0.20}$ \ding{55} & $4.54 \text{\scriptsize$\pm$0.14}$ \ding{55} & $3.05 \text{\scriptsize$\pm$0.24}$ & $2.05 \text{\scriptsize$\pm$0.10}$ \ding{51} & $2.13 \text{\scriptsize$\pm$0.11}$ \ding{51} & $4.79 \text{\scriptsize$\pm$0.21}$ \ding{55} \\
Skill Craft & $0.94 \text{\scriptsize$\pm$0.02}$ \ding{55} & $0.93 \text{\scriptsize$\pm$0.03}$ $\mathbf{=}$ & $0.90 \text{\scriptsize$\pm$0.02}$ $\mathbf{=}$ & $0.90 \text{\scriptsize$\pm$0.02}$ & $0.90 \text{\scriptsize$\pm$0.02}$ $\mathbf{=}$ & $0.91 \text{\scriptsize$\pm$0.02}$ $\mathbf{=}$ & $1.38 \text{\scriptsize$\pm$0.80}$ $\mathbf{=}$ \\
Ailerons ($\times 10^{-4}$) & $1.58 \text{\scriptsize$\pm$0.05}$ \ding{51} & $2.0 \text{\scriptsize$\pm$0}$ $\mathbf{=}$ & $1.71 \text{\scriptsize$\pm$0.04}$ \ding{55} & $1.66 \text{\scriptsize$\pm$0.04}$ & $1.63 \text{\scriptsize$\pm$0.03}$ $\mathbf{=}$ & $2.0 \text{\scriptsize$\pm$0}$ $\mathbf{=}$ & $2.0 \text{\scriptsize$\pm$0}$ $\mathbf{=}$ \\
\midrule
\textbf{W/D/L} & \textbf{11/1/2} & \textbf{0/4/10} & \textbf{0/6/8} &  & \textbf{8/5/1} & \textbf{4/6/4} & \textbf{1/8/5} \\
\bottomrule

\end{tabular}}
\end{table}

\subsection{Complexity}\label{app:interactions}

\begin{figure}[t]
  \centering
  \includegraphics[width=\linewidth]{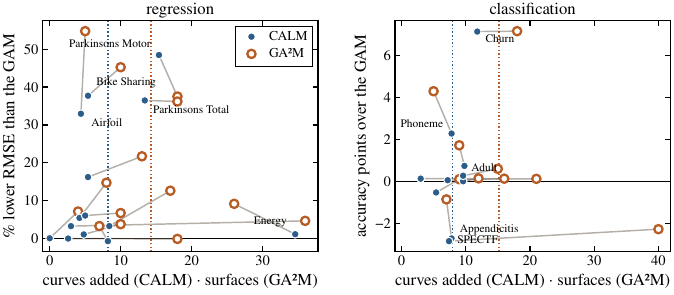}
  \caption{What each model asks a reader to inspect beyond the GAM's one curve per feature, against what it gains over the GAM (24 datasets; means over 5 folds).
    For \method{} the count is its splits, each of which adds one curve; for the \gamm{} it is its pairwise surfaces.
    One dot (\method) and one ring (\gamm) per dataset, joined; dotted lines mark the means.
    \method{} inspects 8.2 added curves on average against 14.3 surfaces in regression and 7.9 against 15.2 in classification, and recovers about 60\,\% of the \gamm's gain where that gain exceeds one point.
    Significance is marked in Tables~\ref{tab:clf-acc} and~\ref{tab:regr-rmse}.}
  \label{fig:tradeoff}
\end{figure}

Figure~\ref{fig:tradeoff} plots the counts of Figure~\ref{fig:dumbbell} against the gains: what each model asks a reader to inspect beyond the GAM, against what it gains over it.
Tables~\ref{tab:classification_interactions} and~\ref{tab:regression_interactions} give the counts per dataset.
For \method{}, the conditional interactions are $\sum_i |S_i|$, where $S_i$ is the set of features that split the tree $T_i$, and the added curves are those beyond the GAM's one per feature, one per split.
On average, \method{} uses 6.8 conditional interactions in classification and 6.6 in regression and adds 7.9 and 8.2 curves, against 15.2 and 14.3 pairs for EB$^2$M; NODE-GA$^2$M uses 97.7 and 88.7 pairs, GAMI-Net its default cap of 20.

\begin{table}[!htbp]
  \centering
  \caption{Complexity per dataset, classification: \method{}'s conditional interactions and added curves, and the pairs of each \gamm{}.}
  \label{tab:classification_interactions}
    \begin{tabular}{lccccc}
      \toprule
      & \multicolumn{2}{c}{\textbf{\method{}}} & & & \\
      \cmidrule(lr){2-3}
      \textbf{Dataset} & \textbf{interactions} & \textbf{curves} & \textbf{EB$^2$M} & \textbf{NODE-GA$^2$M} & \textbf{GAMI-Net} \\
      \midrule
              Adult & $9.2 \text{\scriptsize$\pm$1.5}$ & $9.6 \text{\scriptsize$\pm$1.2}$ & $12.0 \text{\scriptsize$\pm$0.0}$ & $72.6 \text{\scriptsize$\pm$3.2}$ & $20.0 \text{\scriptsize$\pm$0.0}$ \\
        COMPAS & $4.4 \text{\scriptsize$\pm$2.1}$ & $5.4 \text{\scriptsize$\pm$2.2}$ & $9.0 \text{\scriptsize$\pm$0.0}$ & $36.0 \text{\scriptsize$\pm$0.0}$ & $20.0 \text{\scriptsize$\pm$0.0}$ \\
        HELOC & $3.0 \text{\scriptsize$\pm$1.1}$ & $3.0 \text{\scriptsize$\pm$1.1}$ & $21.0 \text{\scriptsize$\pm$0.0}$ & $163.8 \text{\scriptsize$\pm$4.4}$ & $20.0 \text{\scriptsize$\pm$0.0}$ \\
        MIMIC2 & $6.4 \text{\scriptsize$\pm$0.8}$ & $7.2 \text{\scriptsize$\pm$1.0}$ & $16.0 \text{\scriptsize$\pm$0.0}$ & $115.0 \text{\scriptsize$\pm$3.3}$ & $20.0 \text{\scriptsize$\pm$0.0}$ \\
        Appendicitis & $6.6 \text{\scriptsize$\pm$2.7}$ & $7.8 \text{\scriptsize$\pm$1.9}$ & $7.0 \text{\scriptsize$\pm$0.0}$ & $21.0 \text{\scriptsize$\pm$0.0}$ & $6.2 \text{\scriptsize$\pm$8.1}$ \\
        Phoneme & $5.4 \text{\scriptsize$\pm$0.8}$ & $7.8 \text{\scriptsize$\pm$1.5}$ & $5.0 \text{\scriptsize$\pm$0.0}$ & $10.0 \text{\scriptsize$\pm$0.0}$ & $10.0 \text{\scriptsize$\pm$0.0}$ \\
        SPECTF & $7.2 \text{\scriptsize$\pm$2.1}$ & $7.4 \text{\scriptsize$\pm$2.4}$ & $40.0 \text{\scriptsize$\pm$0.0}$ & $273.0 \text{\scriptsize$\pm$14.4}$ & $20.0 \text{\scriptsize$\pm$0.0}$ \\
        Magic & $7.2 \text{\scriptsize$\pm$2.2}$ & $9.8 \text{\scriptsize$\pm$2.1}$ & $9.0 \text{\scriptsize$\pm$0.0}$ & $44.0 \text{\scriptsize$\pm$0.6}$ & $20.0 \text{\scriptsize$\pm$0.0}$ \\
        Bank & $7.6 \text{\scriptsize$\pm$1.5}$ & $9.6 \text{\scriptsize$\pm$1.6}$ & $15.0 \text{\scriptsize$\pm$0.0}$ & $104.2 \text{\scriptsize$\pm$3.2}$ & $20.0 \text{\scriptsize$\pm$0.0}$ \\
        Churn & $10.8 \text{\scriptsize$\pm$0.4}$ & $11.8 \text{\scriptsize$\pm$0.4}$ & $18.0 \text{\scriptsize$\pm$0.0}$ & $137.4 \text{\scriptsize$\pm$2.9}$ & $20.0 \text{\scriptsize$\pm$0.0}$ \\
      \midrule
      \textbf{Avg.} & 6.8 & 7.9 & 15.2 & 97.7 & 17.6 \\
      \bottomrule

    \end{tabular}
\end{table}

\begin{table}[!htbp]
  \centering
  \caption{Complexity per dataset, regression: \method{}'s conditional interactions and added curves, and the pairs of each \gamm{}.}
  \label{tab:regression_interactions}
    \begin{tabular}{lccccc}
      \toprule
      & \multicolumn{2}{c}{\textbf{\method{}}} & & & \\
      \cmidrule(lr){2-3}
      \textbf{Dataset} & \textbf{interactions} & \textbf{curves} & \textbf{EB$^2$M} & \textbf{NODE-GA$^2$M} & \textbf{GAMI-Net} \\
      \midrule
              Bike Sharing & $4.6 \text{\scriptsize$\pm$0.8}$ & $5.4 \text{\scriptsize$\pm$1.2}$ & $10.0 \text{\scriptsize$\pm$0.0}$ & $53.6 \text{\scriptsize$\pm$1.4}$ & $20.0 \text{\scriptsize$\pm$0.0}$ \\
        California Housing & $2.0 \text{\scriptsize$\pm$0.9}$ & $4.2 \text{\scriptsize$\pm$1.5}$ & $8.0 \text{\scriptsize$\pm$0.0}$ & $28.0 \text{\scriptsize$\pm$0.0}$ & $20.0 \text{\scriptsize$\pm$0.0}$ \\
        Parkinsons Motor & $10.6 \text{\scriptsize$\pm$2.4}$ & $15.4 \text{\scriptsize$\pm$3.0}$ & $18.0 \text{\scriptsize$\pm$0.0}$ & $126.2 \text{\scriptsize$\pm$1.7}$ & $20.0 \text{\scriptsize$\pm$0.0}$ \\
        Parkinsons Total & $9.8 \text{\scriptsize$\pm$2.7}$ & $13.4 \text{\scriptsize$\pm$2.7}$ & $18.0 \text{\scriptsize$\pm$0.0}$ & $129.2 \text{\scriptsize$\pm$3.9}$ & $20.0 \text{\scriptsize$\pm$0.0}$ \\
        Seoul Bike & $4.6 \text{\scriptsize$\pm$0.8}$ & $5.4 \text{\scriptsize$\pm$1.2}$ & $13.0 \text{\scriptsize$\pm$0.0}$ & $83.2 \text{\scriptsize$\pm$0.7}$ & $20.0 \text{\scriptsize$\pm$0.0}$ \\
        Wine & $4.2 \text{\scriptsize$\pm$0.7}$ & $4.8 \text{\scriptsize$\pm$0.7}$ & $10.0 \text{\scriptsize$\pm$0.0}$ & $54.6 \text{\scriptsize$\pm$0.5}$ & $20.0 \text{\scriptsize$\pm$0.0}$ \\
        Energy & $31.0 \text{\scriptsize$\pm$2.5}$ & $34.6 \text{\scriptsize$\pm$1.6}$ & $26.0 \text{\scriptsize$\pm$0.0}$ & $192.0 \text{\scriptsize$\pm$4.5}$ & $4.0 \text{\scriptsize$\pm$8.0}$ \\
        CCPP & $0.0 \text{\scriptsize$\pm$0.0}$ & $0.0 \text{\scriptsize$\pm$0.0}$ & $4.0 \text{\scriptsize$\pm$0.0}$ & $6.0 \text{\scriptsize$\pm$0.0}$ & $5.8 \text{\scriptsize$\pm$0.4}$ \\
        Elevators & $5.8 \text{\scriptsize$\pm$1.0}$ & $8.4 \text{\scriptsize$\pm$1.2}$ & $17.0 \text{\scriptsize$\pm$0.0}$ & $115.4 \text{\scriptsize$\pm$2.1}$ & $20.0 \text{\scriptsize$\pm$0.0}$ \\
        No2 & $7.4 \text{\scriptsize$\pm$2.3}$ & $8.2 \text{\scriptsize$\pm$2.2}$ & $7.0 \text{\scriptsize$\pm$0.0}$ & $21.0 \text{\scriptsize$\pm$0.0}$ & $20.0 \text{\scriptsize$\pm$0.0}$ \\
        Sensory & $4.4 \text{\scriptsize$\pm$1.0}$ & $5.0 \text{\scriptsize$\pm$1.3}$ & $10.0 \text{\scriptsize$\pm$0.0}$ & $54.0 \text{\scriptsize$\pm$0.9}$ & $16.0 \text{\scriptsize$\pm$4.9}$ \\
        Airfoil & $3.0 \text{\scriptsize$\pm$0.9}$ & $4.4 \text{\scriptsize$\pm$1.2}$ & $5.0 \text{\scriptsize$\pm$0.0}$ & $10.0 \text{\scriptsize$\pm$0.0}$ & $10.0 \text{\scriptsize$\pm$0.0}$ \\
        Skill Craft & $2.0 \text{\scriptsize$\pm$0.6}$ & $2.6 \text{\scriptsize$\pm$0.5}$ & $18.0 \text{\scriptsize$\pm$0.0}$ & $135.6 \text{\scriptsize$\pm$2.4}$ & $16.0 \text{\scriptsize$\pm$8.0}$ \\
        Ailerons & $2.6 \text{\scriptsize$\pm$0.5}$ & $3.0 \text{\scriptsize$\pm$0.0}$ & $36.0 \text{\scriptsize$\pm$0.0}$ & $232.4 \text{\scriptsize$\pm$5.1}$ & $20.0 \text{\scriptsize$\pm$0.0}$ \\
      \midrule
      \textbf{Avg.} & 6.6 & 8.2 & 14.3 & 88.7 & 16.6 \\
      \bottomrule

    \end{tabular}
\end{table}

\subsection{Sensitivity to hyperparameters}\label{app:sensitivity}
We vary one hyperparameter at a time, the others at their defaults, and compare each variant with the default \method{} by the rule of Section~\ref{sec:empirical-evaluation}.

\textbf{Selection threshold.}
Halving $\gamma$ changes no result and adds about two conditional interactions; doubling it loses on Bike Sharing and Seoul Bike and removes about 1.5 (Table~\ref{tab:sens-gamma}).

\begin{table}[!htbp]
\centering
\small
\begin{minipage}[t]{0.47\textwidth}
\centering
\caption{Halving and doubling $\gamma$ against the default $\gamma = 0.01$: wins, draws and losses of the variant, and the mean change in conditional interactions.}
\label{tab:sens-gamma}
\resizebox{\linewidth}{!}{
\begin{tabular}{llcc}
\toprule
$\gamma$ & task & W / D / L & $\Delta$ interactions \\
\midrule
0.005 & classification & 0 / 10 / 0 & $+1.7$ \\
      & regression     & 0 / 14 / 0 & $+2.1$ \\
0.02  & classification & 0 / 10 / 0 & $-1.8$ \\
      & regression     & 0 / 12 / 2 & $-1.5$ \\
\bottomrule
\end{tabular}}
\end{minipage}\hfill
\begin{minipage}[t]{0.5\textwidth}
\centering
\caption{Share of the GAM-to-XGBoost gap closed by depth (median and quartiles over the datasets with a gap), and the median number of conditional interactions over all datasets.}
\label{tab:sens-depth}
\resizebox{\linewidth}{!}{
\begin{tabular}{cccccc}
\toprule
& \multicolumn{2}{c}{classification (5 of 10)} & & \multicolumn{2}{c}{regression (12 of 14)} \\
\cmidrule(lr){2-3}\cmidrule(lr){5-6}
$d_{\max}$ & gap closed & interactions & & gap closed & interactions \\
\midrule
1 & 17\,\% (10–17) & 2.5 & & 17\,\% (2–42) & 1.9 \\
2 & 31\,\% (27–36) & 6.9 & & 37\,\% (8–55) & 4.5 \\
3 & 41\,\% (35–44) & 10.8 & & 49\,\% (13–64) & 9.6 \\
\bottomrule
\end{tabular}}
\end{minipage}
\end{table}

\textbf{Depth.}\label{app:depth-ladder}
Table~\ref{tab:sens-depth} gives the share of the GAM-to-XGBoost gap that \method{} closes at $d_{\max} = 1, 2, 3$, over the datasets where XGBoost beats the GAM by more than half a fold's standard deviation; Table~\ref{tab:depth-ladders} continues to $d_{\max} = 6$ on four datasets.
Figure~\ref{fig:ladder} shows the same feature of the same dataset at three partition depths. Every level closes more of the gap to the black box and doubles the curves a reader has to hold: two curves answer the questions of \S\ref{subsec:interpretability-properties} at a glance, eight no longer do. The default $d_{\max} = 2$ is where most of the gain has arrived and the plot is still read as a few named regions.

\begin{figure}[!htbp]
  \centering
  \includegraphics[width=\linewidth]{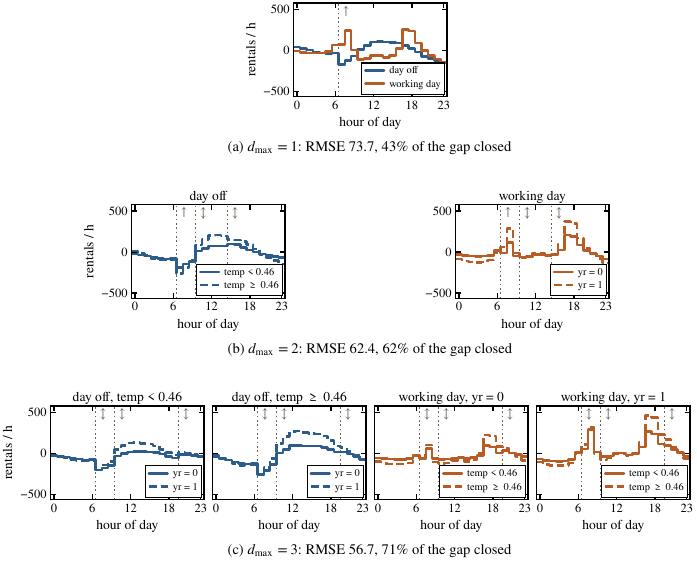}
  \caption{What depth buys: the hour of the Bike Sharing data under \method{} at $d_{\max} = 1$, $2$ and $3$ (fits on all rows; depth 1 is the model of Figure~\ref{fig:hero}, depth 2 the default of Table~\ref{tab:regr-rmse}).
    Each level splits every region once more, giving $2$, $4$ and $8$ curves: (a) shows both, (b) one panel per day type and (c) one panel per region of the first two splits, each panel above the two it splits into; the hue is the day type, the line style the last split; dotted lines mark where another feature's region switches on the hour, each with the sign of its hidden jump (\S\ref{subsec:interpretability-properties}): $\uparrow$~$[36, 334]$ at depth 1; $\uparrow$~$[8, 267]$, $\updownarrow$~$[-159, 63]$ and $\updownarrow$~$[-48, 70]$ at depth 2; all three $\updownarrow$ at depth 3.
    Under each depth, the held-out RMSE of that depth and the share of the GAM\,$\to$\,XGBoost gap it closes, over 5 folds (GAM 100.2, XGBoost 39.4): 43\,\%, 62\,\% and 71\,\%.}
  \label{fig:ladder}
\end{figure}

\begin{table}[!htbp]
\centering
\caption{Depth ladders: score (mean $\pm$ std.\ dev.\ over 5-fold CV; RMSE for Bike Sharing, accuracy otherwise) and conditional interactions (\#) at $d_{\max} = 1, \ldots, 6$. Gap closed at $d_{\max} = 1, \ldots, 6$: Bike Sharing 43, 62, 71, 78, 83, 85\,\%; Bank 17, 36, 44, 55, 53, 64\,\%; Phoneme 10, 31, 41, 43, 45, 59\,\%; Adult has no gap to close.}
\label{tab:depth-ladders}
\resizebox{\textwidth}{!}{
\begin{tabular}{lcccccccc}
\toprule
& \multicolumn{2}{c}{\textbf{Bike Sharing}} & \multicolumn{2}{c}{\textbf{Adult}} & \multicolumn{2}{c}{\textbf{Bank}} & \multicolumn{2}{c}{\textbf{Phoneme}} \\
\cmidrule(lr){2-3}\cmidrule(lr){4-5}\cmidrule(lr){6-7}\cmidrule(lr){8-9}
$d_{\max}$ & RMSE & \# & accuracy & \# & accuracy & \# & accuracy & \# \\
\midrule
GAM & $100.17 \text{\scriptsize$\pm$0.99}$ &  & $0.8699 \text{\scriptsize$\pm$0.0016}$ &  & $0.9027 \text{\scriptsize$\pm$0.0020}$ &  & $0.8266 \text{\scriptsize$\pm$0.0071}$ &  \\
\midrule
1 & $73.74 \text{\scriptsize$\pm$0.96}$ & 1.8 & $0.8702 \text{\scriptsize$\pm$0.0018}$ & 4.0 & $0.9040 \text{\scriptsize$\pm$0.0028}$ & 3.2 & $0.8338 \text{\scriptsize$\pm$0.0108}$ & 2.4 \\
2 & $62.40 \text{\scriptsize$\pm$1.99}$ & 4.6 & $0.8699 \text{\scriptsize$\pm$0.0015}$ & 9.2 & $0.9054 \text{\scriptsize$\pm$0.0023}$ & 7.6 & $0.8494 \text{\scriptsize$\pm$0.0055}$ & 5.4 \\
3 & $56.71 \text{\scriptsize$\pm$1.31}$ & 3.2 & $0.8699 \text{\scriptsize$\pm$0.0026}$ & 11.8 & $0.9060 \text{\scriptsize$\pm$0.0028}$ & 12.2 & $0.8568 \text{\scriptsize$\pm$0.0091}$ & 6.2 \\
4 & $52.54 \text{\scriptsize$\pm$1.53}$ & 5.0 & $0.8691 \text{\scriptsize$\pm$0.0030}$ & 12.6 & $0.9068 \text{\scriptsize$\pm$0.0020}$ & 21.0 & $0.8583 \text{\scriptsize$\pm$0.0121}$ & 7.8 \\
5 & $49.68 \text{\scriptsize$\pm$1.31}$ & 8.0 & $0.8693 \text{\scriptsize$\pm$0.0020}$ & 16.2 & $0.9067 \text{\scriptsize$\pm$0.0015}$ & 24.8 & $0.8595 \text{\scriptsize$\pm$0.0107}$ & 8.0 \\
6 & $48.67 \text{\scriptsize$\pm$1.60}$ & 9.8 & $0.8679 \text{\scriptsize$\pm$0.0014}$ & 20.2 & $0.9075 \text{\scriptsize$\pm$0.0017}$ & 31.0 & $0.8703 \text{\scriptsize$\pm$0.0050}$ & 11.2 \\
\midrule
XGB & $39.35 \text{\scriptsize$\pm$1.38}$ &  & $0.8700 \text{\scriptsize$\pm$0.0025}$ &  & $0.9101 \text{\scriptsize$\pm$0.0037}$ &  & $0.9003 \text{\scriptsize$\pm$0.0037}$ &  \\
\bottomrule

\end{tabular}}
\end{table}

\subsection{Runtime}

Tables~\ref{tab:clf-time-def} and~\ref{tab:regr-time-def} give the fit time of each model. \method{}'s time is end to end, from training the reference model to fitting the shape functions of Step~2. The columns taken from the earlier run were measured on other hardware (Section~\ref{app:datasets-model-details}) and compare only with each other.

\begin{table}[!htbp]
\caption{Runtime (seconds), classification. XGB, EBM, \method{}: one fit (fold 0) on the laptop; NAM, EB$^2$M, NODE-GA$^2$M, GAMI-Net: from an earlier run, mean $\pm$ std.\ dev.\ over 5 folds on the server.}
\label{tab:clf-time-def}
\centering
\small
\begin{tabular}{lccccccc}
\toprule
\textbf{Dataset} & \multicolumn{1}{c}{\textbf{BlackBox}} & \multicolumn{2}{c}{\textbf{GAM}} & \multicolumn{1}{c}{\textbf{\method{}}} & \multicolumn{3}{c}{\textbf{GA$^2$M}} \\
\cmidrule(lr){2-2} \cmidrule(lr){3-4}  \cmidrule(lr){6-8}
& \textbf{XGB} & \textbf{NAM} & \textbf{EBM} &   & \textbf{EB$^2$M} & \textbf{NODE-GA$^2$M} & \textbf{GAMI-Net} \\
\midrule
Adult & $0.3$ & $38 \text{\scriptsize$\pm$2}$ & $87$ & $92$ & $13 \text{\scriptsize$\pm$2}$ & $97 \text{\scriptsize$\pm$9}$ & $718 \text{\scriptsize$\pm$122}$ \\
COMPAS & $0.1$ & $10 \text{\scriptsize$\pm$0.1}$ & $0.6$ & $2.7$ & $1 \text{\scriptsize$\pm$0.04}$ & $51 \text{\scriptsize$\pm$1}$ & $129 \text{\scriptsize$\pm$5}$ \\
HELOC & $0.3$ & $31 \text{\scriptsize$\pm$6}$ & $2.1$ & $25$ & $2 \text{\scriptsize$\pm$0.2}$ & $57 \text{\scriptsize$\pm$0.1}$ & $249 \text{\scriptsize$\pm$60}$ \\
MIMIC2 & $0.3$ & $31 \text{\scriptsize$\pm$0.3}$ & $3.0$ & $24$ & $5 \text{\scriptsize$\pm$1}$ & $69 \text{\scriptsize$\pm$9}$ & $359 \text{\scriptsize$\pm$32}$ \\
Appendicitis & $0.3$ & $5 \text{\scriptsize$\pm$0.1}$ & $2.5$ & $1.2$ & $1 \text{\scriptsize$\pm$0.5}$ & $24 \text{\scriptsize$\pm$0.1}$ & $6 \text{\scriptsize$\pm$5}$ \\
Phoneme & $0.5$ & $6 \text{\scriptsize$\pm$1}$ & $6.5$ & $5.7$ & $4 \text{\scriptsize$\pm$1}$ & $61 \text{\scriptsize$\pm$4}$ & $166 \text{\scriptsize$\pm$38}$ \\
SPECTF & $0.1$ & $32 \text{\scriptsize$\pm$0.3}$ & $1.8$ & $4.2$ & $4 \text{\scriptsize$\pm$5}$ & $25 \text{\scriptsize$\pm$0.2}$ & $180 \text{\scriptsize$\pm$13}$ \\
Magic & $0.3$ & $28 \text{\scriptsize$\pm$18}$ & $8.6$ & $23$ & $7 \text{\scriptsize$\pm$4}$ & $98 \text{\scriptsize$\pm$7}$ & $489 \text{\scriptsize$\pm$63}$ \\
Bank & $0.5$ & $44 \text{\scriptsize$\pm$1}$ & $13$ & $49$ & $13 \text{\scriptsize$\pm$1}$ & $98 \text{\scriptsize$\pm$17}$ & $994 \text{\scriptsize$\pm$265}$ \\
Churn & $0.3$ & $19 \text{\scriptsize$\pm$3}$ & $1.8$ & $11$ & $3 \text{\scriptsize$\pm$2}$ & $56 \text{\scriptsize$\pm$2}$ & $257 \text{\scriptsize$\pm$32}$ \\
\midrule
\textbf{Avg.} & 0.3 & 24 & 13 & 24 & 5.3 & 64 & 355 \\
\bottomrule

\end{tabular}
\end{table}

\begin{table}[!htbp]
\caption{Runtime (seconds), regression. Sources as in Table~\ref{tab:clf-time-def}.}
\label{tab:regr-time-def}
\centering
\small
\begin{tabular}{lccccccc}
\toprule
\textbf{Dataset} & \multicolumn{1}{c}{\textbf{BlackBox}} & \multicolumn{2}{c}{\textbf{GAM}} & \multicolumn{1}{c}{\textbf{\method{}}} & \multicolumn{3}{c}{\textbf{GA$^2$M}} \\
\cmidrule(lr){2-2} \cmidrule(lr){3-4}  \cmidrule(lr){6-8}
& \textbf{XGB} & \textbf{NAM} & \textbf{EBM} &   & \textbf{EB$^2$M} & \textbf{NODE-GA$^2$M} & \textbf{GAMI-Net} \\
\midrule
Bike Sharing & $0.3$ & $18 \text{\scriptsize$\pm$1}$ & $6.2$ & $20$ & $19 \text{\scriptsize$\pm$3}$ & $187 \text{\scriptsize$\pm$31}$ & $299 \text{\scriptsize$\pm$33}$ \\
California Housing & $0.3$ & $15 \text{\scriptsize$\pm$0.1}$ & $16$ & $21$ & $16 \text{\scriptsize$\pm$0.5}$ & $200 \text{\scriptsize$\pm$46}$ & $677 \text{\scriptsize$\pm$147}$ \\
Parkinsons Motor & $0.5$ & $21 \text{\scriptsize$\pm$6}$ & $45$ & $89$ & $44 \text{\scriptsize$\pm$4}$ & $193 \text{\scriptsize$\pm$34}$ & $536 \text{\scriptsize$\pm$88}$ \\
Parkinsons Total & $0.5$ & $18 \text{\scriptsize$\pm$0.2}$ & $70$ & $82$ & $40 \text{\scriptsize$\pm$5}$ & $227 \text{\scriptsize$\pm$76}$ & $543 \text{\scriptsize$\pm$79}$ \\
Seoul Bike & $0.3$ & $15 \text{\scriptsize$\pm$0.2}$ & $4.4$ & $12$ & $12 \text{\scriptsize$\pm$4}$ & $127 \text{\scriptsize$\pm$18}$ & $268 \text{\scriptsize$\pm$53}$ \\
Wine & $0.3$ & $12 \text{\scriptsize$\pm$1}$ & $1.5$ & $6.5$ & $3 \text{\scriptsize$\pm$0.4}$ & $62 \text{\scriptsize$\pm$10}$ & $187 \text{\scriptsize$\pm$20}$ \\
Energy & $0.5$ & $43 \text{\scriptsize$\pm$0.4}$ & $54$ & $118$ & $97 \text{\scriptsize$\pm$15}$ & $206 \text{\scriptsize$\pm$57}$ & $253 \text{\scriptsize$\pm$214}$ \\
CCPP & $0.2$ & $6 \text{\scriptsize$\pm$0.1}$ & $42$ & $42$ & $14 \text{\scriptsize$\pm$2}$ & $133 \text{\scriptsize$\pm$18}$ & $88 \text{\scriptsize$\pm$8}$ \\
Elevators & $0.3$ & $25 \text{\scriptsize$\pm$0.4}$ & $5.8$ & $35$ & $43 \text{\scriptsize$\pm$8}$ & $181 \text{\scriptsize$\pm$38}$ & $773 \text{\scriptsize$\pm$269}$ \\
No2 & $0.2$ & $5 \text{\scriptsize$\pm$0.04}$ & $0.7$ & $1.9$ & $2 \text{\scriptsize$\pm$0.1}$ & $26 \text{\scriptsize$\pm$0.1}$ & $78 \text{\scriptsize$\pm$1}$ \\
Sensory & $0.1$ & $10 \text{\scriptsize$\pm$1}$ & $1.0$ & $2.8$ & $2 \text{\scriptsize$\pm$0.4}$ & $27 \text{\scriptsize$\pm$0.5}$ & $74 \text{\scriptsize$\pm$16}$ \\
Airfoil & $0.4$ & $4 \text{\scriptsize$\pm$0.03}$ & $2.8$ & $3.7$ & $7 \text{\scriptsize$\pm$1}$ & $51 \text{\scriptsize$\pm$12}$ & $76 \text{\scriptsize$\pm$34}$ \\
Skill Craft & $0.5$ & $17 \text{\scriptsize$\pm$0.2}$ & $2.1$ & $7.9$ & $1 \text{\scriptsize$\pm$0.1}$ & $44 \text{\scriptsize$\pm$0.1}$ & $141 \text{\scriptsize$\pm$52}$ \\
Ailerons & $0.4$ & $53 \text{\scriptsize$\pm$1}$ & $4.0$ & $69$ & $3 \text{\scriptsize$\pm$0.1}$ & $126 \text{\scriptsize$\pm$36}$ & $619 \text{\scriptsize$\pm$134}$ \\
\midrule
\textbf{Avg.} & 0.3 & 19 & 18 & 37 & 22 & 128 & 329 \\
\bottomrule

\end{tabular}
\end{table}

\clearpage

\section{User Study}\label{app:user-study}

\textbf{Study overview.}
We conducted a within-subjects user study ($N=25$) to evaluate the practical interpretability of CALMs. Participants ranged from novices to experts in machine learning and data analysis. The survey was anonymous, collected no personal data, asked for informed consent and was not paid; under our institution's policy, such a study is exempt from ethics-committee review.
The CALM questions came first, with a warm-up on a synthetic model, and the GA$^2$M appeared only on Bike Sharing; presentation order and dataset complexity were not counterbalanced, so the comparison should be read with this caveat.

\textbf{Participant expertise.}
Participants self-reported their level of expertise using a 5-point scale ranging from no experience to expert.

\begin{figure}[h]
    \centering
    \includegraphics[width=\linewidth]{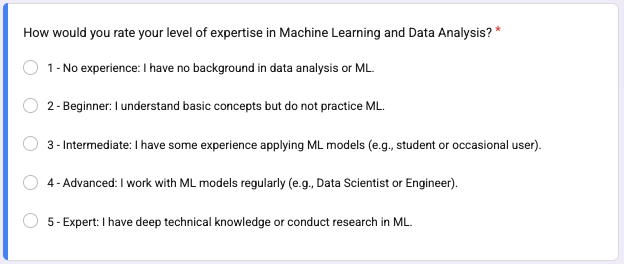}
    \caption{Self-reported expertise question used to assess participant background.}
    \label{fig:expertise}
\end{figure}

\textbf{Interpretation protocol.}
Before starting the tasks, participants were provided with standardized instructions on how to interpret conditional feature plots. These instructions explained how to:
(i) select the appropriate curve based on feature conditions,
(ii) read feature contributions directly from the plot, and
(iii) account for interaction-induced discontinuities when estimating changes.

\begin{figure}[h]
    \centering
    \includegraphics[width=\linewidth]{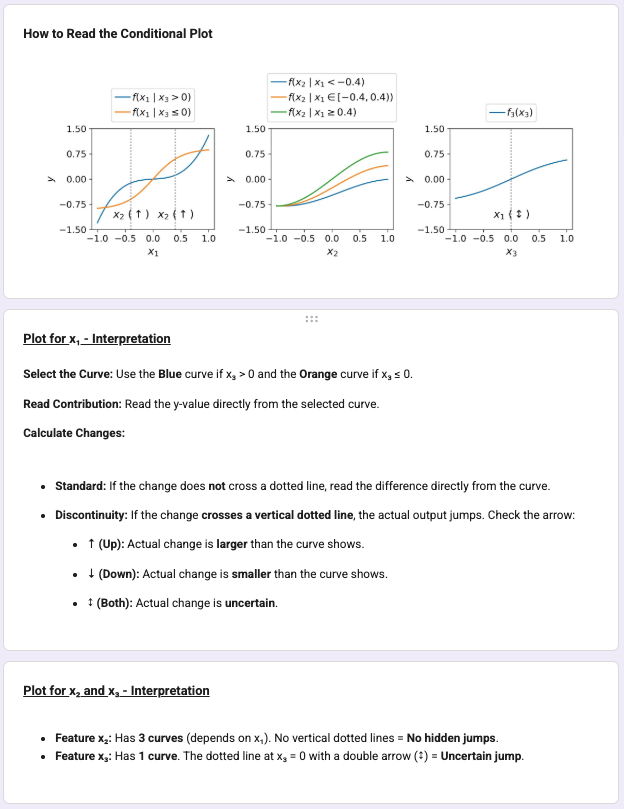}
    \caption{Instruction panel provided to participants for interpreting conditional feature plots.}
    \label{fig:instructions}
\end{figure}

\textbf{Task design.}
Participants completed tasks targeting two aspects of interpretability: estimating feature contributions for specific inputs and reasoning about how predictions change under feature perturbations~\cite{lage2019human, hase2020evaluating}.

\textbf{Example questions.}
Figure~\ref{fig:question_contribution} shows a representative example of the questions used in the study. The full questionnaire is available at \url{https://forms.gle/tTippi5eJ6cZHrfy9}.

\begin{figure}[h]
    \centering
    \includegraphics[width=\linewidth]{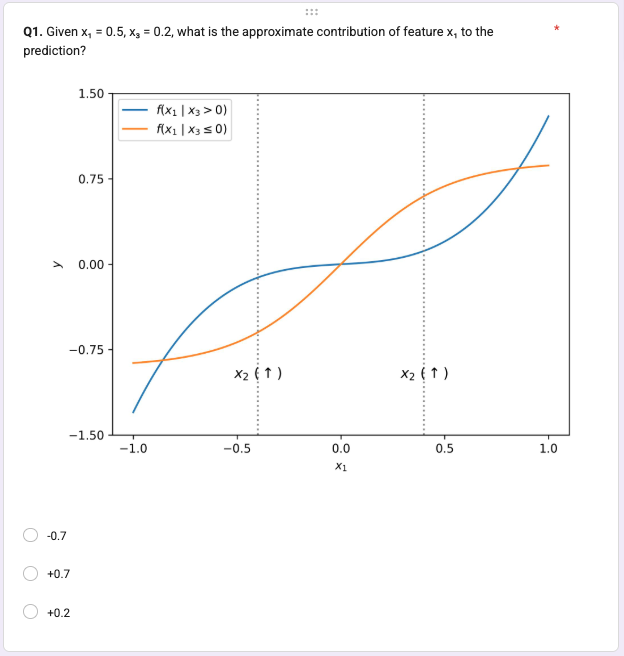}
    \caption{Example question requiring estimation of feature contribution from a conditional plot.}
    \label{fig:question_contribution}
\end{figure}

\textbf{Evaluation metrics.}
Every question is multiple choice with one correct option, so grading is automatic: an answer is correct if and only if it selects that option.
We evaluate performance using three metrics: task accuracy (percentage of correct answers), accuracy on prediction change tasks, and user experience, measured via self-reported ease of use and confidence.

\textbf{Results.}
On the five warm-up questions on the synthetic model, participants answered $83.2\%$ correctly with CALM. On real-world tasks, CALM significantly outperformed heatmap-based representations in overall task performance ($54\%$ vs.\ $34\%$, $p < 0.05$), driven by the question of how long the hour takes to lower the prediction by 50 rentals ($40\%$ vs.\ $12\%$, $p < 0.05$; $68\%$ vs.\ $56\%$ on the other). The advantage holds among the 13 participants who rated themselves advanced or expert ($65.4\%$ vs.\ $30.8\%$, Wilcoxon $p = 0.031$), and accuracy with the GA$^2$M does not rise with self-reported expertise (Spearman $\rho = -0.29$, $p = 0.14$).

Participants also rated CALM as easier to use ($p < 0.01$) and more confidence-inducing ($p < 0.001$), indicating improved usability alongside stronger objective performance.

\clearpage
\section{Broader Impact}\label{app:broader-impact}
\method{} aims at models that are transparent and easy to audit: each feature keeps one plot, and each curve states the region where it applies, which supports decision making in high-stakes domains such as credit scoring and healthcare. The same plots should be used with care: they show what \emph{the model} does, not how \emph{the data} were generated, and conclusions about the process that generated the data need further assumptions~\cite{freiesleben2024scientific}.

\newpage

\end{document}